\documentclass[11pt]{article}
\usepackage{amsfonts,amsmath,bm, xspace,caption}
\usepackage{multirow,graphicx, algorithm,algorithmic,rotating}
\usepackage{url,cite}
\usepackage{fullpage}

\newtheorem{theorem}{Theorem}[section]
\newtheorem{lemma}[theorem]{Lemma}
\newtheorem{corollary}[theorem]{Corollary}
\newtheorem{proposition}[theorem]{Proposition}
\newtheorem{definition}[theorem]{Definition}

\def \endprf{\hfill {\vrule height6pt width6pt depth0pt}\medskip}
\newenvironment{proof}{\noindent {\bf Proof} }{\endprf\par}

\newcommand{\eat}[1]{}

\usepackage[tight,footnotesize]{subfigure}
\usepackage{amssymb, epsfig,bm}

\newcommand{\G}{\mathcal{G}}
\newcommand{\A}{\mathcal{A}}
\newcommand{\N}{\mathcal{N}}
\newcommand{\R}{\mathbb{R}}
\newcommand{\E}{\mathbb{E}}
\newcommand{\K}{\mathcal{K}}

\newcommand{\vct}[1]{\bm{#1}}
\newcommand{\mtx}[1]{\bm{#1}}
\newcommand{\vx}{\vct{x}}
\newcommand{\vy}{\vct{y}}
\newcommand{\vz}{\vct{z}}
\newcommand{\vv}{\vct{v}}
\newcommand{\va}{\vct{a}}
\newcommand{\vw}{\vct{w}}
\newcommand{\vr}{\vct{r}}

\newcommand{\x}{\vct{x}}
\newcommand{\y}{\vct{y}}
\newcommand{\z}{\vct{z}}

\newcommand{\mk}{\mtx{K}}
\newcommand{\mS}{\mtx{\Sigma}}
\newcommand{\mV}{\mtx{V}}
\newcommand{\mU}{\mtx{U}}
\newcommand{\mP}{\mtx{P_S}}
\newcommand{\br}{\bar{\vr}}
\newcommand{\bw}{\bar{\vw}}
\newcommand{\bv}{\bar{\vv}}

\title{Signal Recovery in Unions of Subspaces with Applications to Compressive Imaging}

\author{
Nikhil Rao$^\dagger$, Benjamin Recht$^*$ and Robert D. Nowak$^\dagger$\\
$^*$ Computer Sciences Department, University of Wisconsin-Madison\\
$^\dagger$ Electrical and Computer Engineering, University of Wisconsin-Madison}

\date{September 2012}

\begin{document}

\maketitle

\vspace{-0.3in}

\begin{abstract}
In applications ranging from communications to genetics, signals can be modeled as lying in a union of subspaces. Under this model, signal coefficients that lie in certain subspaces are active or inactive together. The potential subspaces are known in advance, but the particular set of subspaces that are active (i.e., in the signal support) must be learned from measurements. We show that exploiting  knowledge of subspaces can further reduce the number of measurements required for exact signal recovery, and derive universal bounds for the number of measurements needed. The bound is universal in the sense that it only depends on the number of subspaces under consideration, and their orientation relative to each other. The particulars of the subspaces (e.g., compositions, dimensions, extents, overlaps, etc.) does not affect the results we obtain.  In the process, we derive sample complexity bounds for the special case of the group lasso with overlapping groups (the latent group lasso), which is used  in a variety of applications. Finally, we also show that wavelet transform coefficients of images can be modeled as lying in groups, and hence can be efficiently recovered using group lasso methods.  
\end{abstract}


\section{Introduction}
In many fields such as genetics, image processing, and machine learning, one is faced with the task of recovering very high dimensional signals from relatively few measurements. In general this is not possible, but fortunately many real world signals are, or can be transformed to be, sparse, meaning that only a small fraction of signal coefficients is non-zero. Compressed Sensing \cite{CRT,donoho} allows us to recover sparse, high dimensional signals with very few measurements as compared to the ambient signal dimension. In fact,  results indicate that one only needs $\mathcal{O}(s\cdot \log p)$ random measurements to exactly recover an $s$ sparse signal of length $p$.

In many applications however, one not only has knowledge about the sparsity of the signal, but some additional information about the structure of the sparsity pattern as well:

\begin{enumerate} 
\item In genetics, the genes are arranged into pathways/clusters, and genes belonging to the same pathway are often active/inactive in a group \cite{pathway, hassibigene}.
\item In image processing, the wavelet transform coefficients can be modeled as belonging to a tree, with parent-child coefficients simultaneously being large or small \cite{crouse98, romberg}.
\item In wideband spectrum sensing applications, the spectrum typically displays clusters of non-zero frequency coefficients, each corresponding to a narrowband transmission \cite{analogCS}
\item In applications in analog compressed sensing \cite{eldaranalog1, eldaranalog2}, the signals can be expressed as lying in a union of shift invariant subspaces.
\item In reconstruction of signals having a finite rate of innovation \cite{bluinnov1, bluinnov2}, non zeros are known to be clustered spatially, corresponding to objects in a scene for example.
\end{enumerate}

In cases such as these, the sparsity pattern can be represented lying in a union of certain subspaces (e.g., coefficients in certain pathways, tree branches, frequency bands, or clusters). This knowledge about the signal structure can help further reduce the number of measurements one needs to exactly recover the signal. In this paper, we derive bounds on the number of random i.i.d. Gaussian measurements needed to exactly recover a sparse signal when its pattern of sparsity lies in a union of subspaces, based on solving a \emph{convex} recovery algorithm. This characterization specializes to the latent group lasso, introduced in \cite{jacob, latent, percival}, wherein the sparsity pattern can be expressed as lying in a union of groups.

We analyze the recovery problem using a random Gaussian measurement model. We emphasize that although the derivation assumes the measurement matrix to be Gaussian, it can be extended to \emph{any} subgaussian case, by paying a small constant penalty, as shown in \cite{subgaussian}. We restrict ourselves to the Gaussian case here since it highlights the main ideas and keeps the analysis as simple as possible.

\subsection{Prior Work}
To the best of our knowledge, these results are new and distinct from prior theoretical characterizations of group lasso and the general union of subspace methods. Sampling theorems for unions of subspaces have been considered in \cite{ludo, BluDav}, where the authors show that the sample complexity depends logarithmically on the number of subspaces under consideration. In \cite{eldaruos}, the authors also propose a greedy scheme to recover signals that lie in such unions. The authors in \cite{huang09} derive information theoretic bounds for the number of measurements needed for a variety of signal ensembles, including trees. In \cite{modelbased,model09}, the authors show that one needs far fewer measurements when the signal can be expressed as lying in a union of subspaces, and explicit bounds are derived when using a modified version of CoSaMP \cite{cosamp} to recover the signal. Asymptotic consistency results are derived for the group lasso \cite{yuanlin} when the groups partition the space of variables in \cite{bachConsistency}. Similarly, in \cite{zhang09}, the authors again consider the groups to partition the space, and derive conditions for recovery using the group lasso. In \cite{jenatton, jenatton10}, the authors derive consistency results for the group lasso under arbitrary groupings of variables. Also,  in \cite{Mest}, the authors consider overlapping groups and derive sample bounds. The authors in \cite{jacob}  derive consistency results in an asymptotic setting, for the group lasso with overlap, but do not provide exact recovery results. The group lasso with overlap (called the latent group lasso in \cite{latent}) is analyzed in detail in \cite{latent, percival}. 

Non-greedy schemes have not been developed to handle the case of the union of subspaces, nor the latent group lasso. Although group lasso with overlapping groups have been considered in the past, it yields vectors whose support can be expressed as a complement of a union of groups, while we consider cases where we require the support to lie in a union of groups, a distinction made in \cite{jacob}. In the applications considered above, however, it is imperative that we recover a sparsity pattern that lies in a union of groups (or subspaces in more generality). 


\subsection{Our contributions}

To derive our results, we appeal to the notion of restricted minimum singular values of an operator. The restricted minimum singular value (or equivalently the restricted eigenvalues of the gram matrix of the operator) conditions are weaker than the well known Restricted Isometry Conditions of operators, and have been used and studied in \cite{Bickel, tangCMSV}, among others.

We bound number of measurements needed for exact recovery with two terms. One term ($kB$) grows linearly in the total number of non-zero coefficients (with a constant of proportionality). This is close to the bare minimum of one measurement per non-zero component. Intuitively, this term corresponds to the magnitude estimation once the locations of the non zero components have been determined. The other term logarithmically depends on the number of subspaces under consideration and their relative orientations, and not the particulars of the subspaces (e.g., compositions, dimensions, extents, etc.). In particular, the subspaces need not be disjoint. Intuitively, this term corresponds to the price we pay for the detection of the non zeros in the signal (the active subspaces). The degree to which subspaces overlap, remarkably, has no effect on our bounds. In this regard, our bounds can be termed to be \emph{universal}. This is somewhat surprising since overlapping subspaces are strongly coupled in the observations, tempting one to suppose that overlap may make recovery more challenging.

Our main theoretical result shows that for signals with support on $k$ of $M$ possible subspaces, exact recovery is possible from $C_1*(\sqrt{2\log(M-k)} + \sqrt{B})^2k + C_2*kB$ measurements using a latent group lasso type algorithm, B being the maximum subspace dimension. The constants $C_1$ and $C_2$ depend on the relative angle between subspaces, and will be explicitly derived in the sequel. Note that the bound depends on the sparsity $s$ of the signal via the $kB$ term. The latent group lasso reduces to a  special case of our result. We will routinely compare the performance of the group lasso to the standard lasso, to study the effects of overlap between subspaces on the actual number of measurements needed to exactly recover a signal. For the lasso bound, we will use the one derived in \cite{venkat}: $(2s + 1)\log(p-s)$. Assuming that $M = \mathcal{O}(poly(p))$, our bound is roughly $k\log(p) + kB$. For the same problems, the lasso which ignores the group structure of the sparse signal components would require approximately $kB\log(p)$ measurements. Hence, taking advantage of the subspace structure will allow us to take fewer measurements to reconstruct the signal. 

Note that in this work, the subspaces can be arbitrary, and we make \emph{no} assumptions about their nature, except that they are known in advance. In short, we derive bounds for any generic union of subspaces, whether they overlap or form a partition of the ambient high dimensional space. Note that, when we say ``do not overlap", we mean that the intersection of the subspaces is $\{\vct0\}$, the all zeros vector.

We then propose a novel way to model wavelet coefficients of images in this framework, and show that we perform at least as well as several other state of the art methods in compressive imaging.

To summarize, our contributions in this paper are as follows:

\begin{enumerate}
\item We derive non asymptotic sample complexity bounds in a compressive-sensing framework when the measurement matrix is i.i.d. Gaussian, and the signal can be expressed as lying in a sparse union of finite dimensional subspaces.
\item We show that our bound holds regardless of the nature of overlap between subspaces. In this sense, the bounds we derive are universal. 
\item We show that the group lasso with overlapping groups is a special case of the general result that we derive
\item We propose a new method to model wavelet coefficients of signals, so that one can use convex optimization algorithms to recover the signal exactly, or with high fidelity in the presence of noise. 
\item An extensive series of experiments verify the theory.
\end{enumerate}

The rest of this paper is organized as follows: In section \ref{sec:setup} we set up the problem of recovering a signal lying in a union of subspaces, and present some preliminaries that will be useful in deriving the bounds. Our result for the sample complexity for the union of subspaces is derived in section \ref{sec:result}. We then extend the framework for the case of group lasso with overlapping groups in section \ref{sec:glassoresult}. Section \ref{sec:approx} extends our results to approximately sparse signals. In section \ref{sec:icip} we propose a novel framework for modeling wavelet transform coefficients of images, that makes use of the concepts we present in the sections that precede it. Experimental validation is provided in section \ref{sec:expts}. Finally, we conclude our paper in section \ref{sec:conc}, and present avenues for future work. 

\section{The Union of Subspaces Model}
\label{sec:setup}
In the course of the next sections, we derive measurement bounds for the exact recovery of a signal lying in a  sparse union of finite dimensional subspaces, and the robust recovery of one that is approximately sparse. In this section, we will argue as to why exact recovery of the signal corresponds to the minimization of the atomic norm of the signal, with the atoms obeying certain properties governed by the signal structure. Before we do so, we dispense with the notations. 

\subsection{Notations}
Consider a signal of length $p$, that is $s$ sparse. Note here that in case of multidimensional signals like images, we assume they are vectorized to have length $p$.

Suppose we are given a set of bases for $M$ subspaces
\[
\K = \{ K_1, K_2, \ldots, K_M \}
\]
Let the dimensions of each subspace be given by $d_1, d_2, \ldots, d_M$, with $B = \max_i d_i$. So, $K_i \in \R^{p \times d_i}$. We assume that $span[K_1K_2 \ldots K_M] = \R^p$. Without loss of generality, assume each $K_i$ to be orthonormal. If not, we can perform the Gram Schmidt procedure to orthonormalize them.  The subspaces can be overlapping or non overlapping. When we say two subspaces do not overlap, we mean that 
\[
\mathrm{span}(K_i) \cap \mathrm{span}(K_j) = \{0\}
\]

Also, the subspaces may or may not be perpendicular to each other. Two subspaces are perpendicular if 
\begin{equation} \label{eqperp}
| \langle k_a, k_b\rangle | = 0  ~\ \forall k_a \in A_i \text{ and } \forall k_b \in A_j\end{equation}
where 
\[ A_i = \mbox{span}(K_i)\backslash (\mbox{span}(K_i) \cap \mbox{span}(K_j))\]
and
\[ A_j = \mbox{span}(K_j)\backslash( \mbox{span}(K_i) \cap \mbox{span}(K_j)) \]

The notion of subspaces being perpendicular to each other is clarified in Fig. \ref{perpsubs}. It is fairly obvious that the more perpendicular the subspaces are to each other, the more separated they are, and hence the easier to is to distinguish between which among the subspaces is active.

\begin{figure}[!h]
\centering
\includegraphics[scale = 0.8]{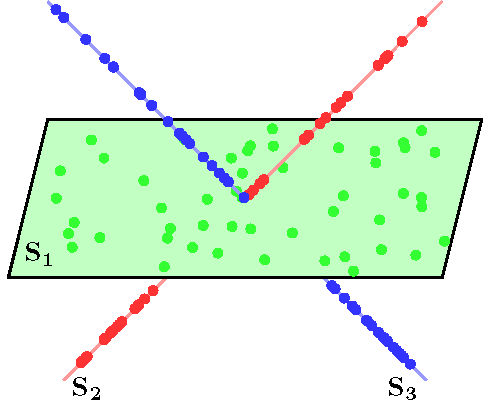}
\caption{Three subspaces, commonly encountered in subspace clustering methods. Subspaces S2 and S3 are perpendicular to each other. However, S1 and S2 are not perpendicular, nor are S1 and S3}
\label{perpsubs}
\end{figure}

We denote vectors by bold lowercase letters ($\va, \vv$ etc.), and matrices by bold uppercase letters ($\mtx{M}$, etc.). Subscripts following a vector denotes a particular index of the vector. For any vector $\vv$, we will routinely use the following decomposition:
\[
\vv = \sum_{i = 1}^M K_i \bv^i
\]
where $\bv^i \in \R^{d_i}$.
The decomposition holds since we assume $\mathrm{span}([\mk_1 \mk_2, \ldots ,\mk_M]) = \R^p$. Superscripts following a matrix will denote the sub matrix whose columns are the columns indexed by the superscript. $|\cdot|$ denotes the cardinality of a set. 

We let $\vx^\star$ be the (subspace sparse) signal to be recovered, whose non zero coefficients lie in $k$ of the $M$ subspaces $\K^\star \subset \K$, with $k << M$.  Formally, noting that $\vx^\star = \sum_i K_i \bar{x}^{\star i}$, 
\[
\K^\star = \left\{ K_i \in \K: \| \bar{x}^{\star i}\| \neq 0 \right\}
\]
We then have $|\K^\star | = k$. Let the indices of then active subspaces be given by $J$. That is, 
\[
J \subset \{1,2, \ldots ,p\} : \{ j \in J \iff K_j \in \K^\star \}
\]
Later in the paper, we will also consider approximately sparse signals. We let $\mtx{\Phi}_{n\times p}$ be a measurement matrix consisting of i.i.d. Gaussian entries of mean 0 and variance $\frac1n$ so that every column is a realization of an i.i.d. Gaussian length $n$ vector with covariance matrix $\frac1n \mtx{I}$. We denote the observed vector by $\vy \in \R^n ~\ : \vy = \mtx\Phi \vx^\star$. The absence of a subscript following a norm $\| \cdot \|$ implies the $\ell_2$ norm. The dual norm of $\|\cdot\|_p$ is denoted by $\|\cdot\|_p^*$. The convex hull of a set of points $S$ is denoted by $\mathrm{conv}(S)$. We let $\sigma_i(\mtx{M})$ denote the $i^{th}$ singular value of a matrix $\mtx{M}$. We define
\[
\mtx{K^\star} = [K_{j_1} K_{j_2} \ldots] ~\ \forall j_i \in J
\]
and
\[
\mtx K = [K_1 K_2 \ldots K_M]
\]
Finally, let $\kappa(\mk)$ be the condition number of $\mk$.



\subsection{Atoms, Atomic Set and the Atomic Norm}

To begin with, let us formalize the notion of atoms and the atomic norm of a signal (or vector). We will restrict our attention to  signals in $\R^p$ that can be expressed as  lying in a sparse union of subspaces, though the same concepts can be extended to other spaces as well. We assume that $\vx \in \R^p$ can be decomposed as :
\[
\vx = \sum_{i = 1}^k c_i \va^i, ~\ c_i \geq 0 
\]
The vectors $\va^i \in \R^p$ are called \emph{atoms}, and form the basic building blocks of any signal, which can be represented as a conic combination of the atoms. Note that the sum notation, rather than the integral notation, implies that only a countable number of coefficients can be non-zero. We denote $\A = \{ \va\}$ to be the \emph{atomic set}. Given a vector $\vx\in\R^p$ and an atomic set, we define the \emph{atomic norm} as
\begin{equation}
\label{anormdef}
||\vx||_{\A} = \inf \left\{ \sum_{\va \in \A} c_a : \vx =  \sum_{\va \in \A} c_a\va, ~\  c_a \geq 0   ~\ \forall \va \in \A \right\}
\end{equation}
The atomic decomposition of the signal yields a representation of a signal in terms of some predefined atoms. Usually, few atoms used in a representation indicates a ``simpler" representation. Hence, to obtain a ``simple" representation of a vector, we look to minimize the atomic norm subject to constraints (\ref{minAnorm}): 
\begin{equation}
\label{minAnorm}
\hat{\x} = \underset{\x \in \R^p}{\operatorname{argmin}} ~\ ||\x||_\A  ~\ \textbf{s.t. } \y = \mtx\Phi \x
\end{equation}

When the measurements are corrupted with noise $\vct\theta$, such that $\| \vct\theta\| \leq \epsilon$, the atomic norm minimization problem becomes:
 \begin{equation}
\label{minAnormnoise}
\hat{\x} = \underset{\x \in \R^p}{\operatorname{argmin}}~\ ||\x||_\A  ~\ \textbf{s.t. } \|\y - \mtx\Phi \x\|^2 \leq \epsilon^2
\end{equation}

Indeed, when the atoms are merely the canonical basis in $\R^p$, the atomic norm reduces to the standard $\ell_1$ norm, and minimization of the atomic norm yields the well known \emph{lasso} procedure \cite{tibshirani}. 

Assuming we are aware of the subspaces $\K$, we now proceed to define the atomic set and the corresponding atomic norm for our framework. Let
\[
A_i = \{ \va \in \R^p ~\ :  \exists \vct\alpha^i ~\ : \va = K_i \vct\alpha^i, ~\ \|\va\| = \|\vct\alpha^i\| = 1 \}
\]
\begin{equation}
\label{aset}
\A  = \cup_{i = 1}^M A_i
\end{equation}

The sub vectors $\va \in A_i$ form the boundary of the unit sphere restricted to the span of $K_i$.

The atomic norm of a vector $x \in \R^p$ is given by
\begin{equation}
\label{anormdef}
\| \x \|_\A = \min_{\x = \sum_{\va} c_i \va} \sum c_i ~\ c_i \geq 0
\end{equation}

\begin{lemma}
\label{lemanormform}
The atomic norm for a signal $\x \in \R^p$ lying in a union of subspaces is given by
\[
\|\x\|_\A = \min_{\x = \sum_i K_i \vct\alpha^i} \sum_i \| \vct\alpha^i \|
\]
\end{lemma}
\begin{proof}
The result follows by substituting
\[
c_i = \| K_i \vct\alpha^i \| = \|\vct\alpha^i\|
\] 
and
\[
\va = K_i \left( \frac{\vct\alpha_i}{\| K_i \vct\alpha^i \| } \right) = K_i \left(\frac{\vct\alpha_i}{\|\vct\alpha^i \| }\right)
\]
\end{proof}

The notion of the atomic set and the corresponding atomic norm for the union of subspaces model is made clear in Fig \ref{balls}. The figure shows the atomic norm balls, given by the convex hull of the atomic set $\A$. 

\begin{figure}[!h]
\begin{center}
\subfigure[Norm ball for Perpenducular subspaces]{
\includegraphics[scale = 0.25]{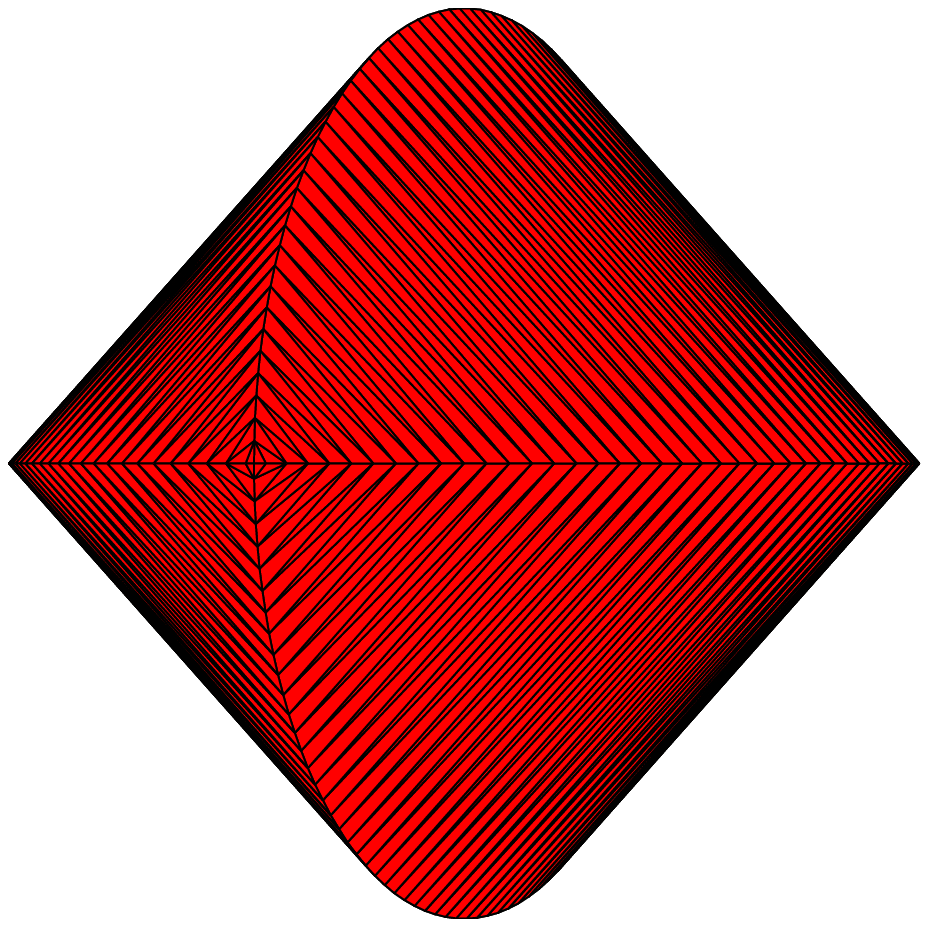}
\label{fig:glassoball}}
\subfigure[Norm ball for Non-perpendicular subspaces]{
\includegraphics[scale = 0.25]{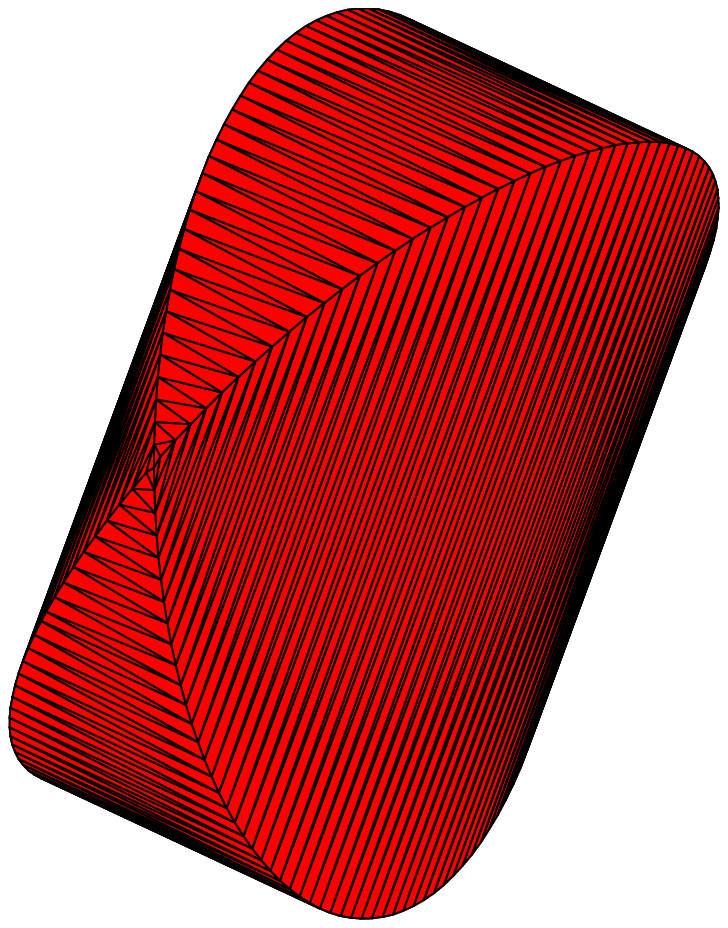}
\label{fig:badglassoball}}
\caption{Atomic norm balls for a pair of perpendicular and non-perpendcular subspaces. The atomic set corresponds to the union of boundaries of the uint disks. Each disk corresponds to a particular $A_i$ in (\ref{aset})}
\label{balls}
\end{center}
\end{figure}

Also note that we can directly compute the dual of the atomic norm from the set of atoms
\begin{align}
\notag
\|\vx\|^*_\A &= \sup_{\va} \langle \vx , \va \rangle \\
\label{dualdef}
&= \max_{i = 1,2,\ldots M} \| (K_i)^T \vx  \|
\end{align}
That is, the dual norm is the maximum over the norms of the projections of $\vx$ onto the different subspaces, noting that $[(K_i)^TK_i]^{-1} = \mtx{I}$.
The dual norm will be useful in our derivations below.

\subsection{Gaussian Widths and Exact Recovery}

Following \cite{venkat}, we define the \emph{tangent cone} and \emph{normal cone} at $\x^\star$ with respect to $conv(\A)$ under $||\x||_\A$ as \cite{wets}:
\begin{align}
\label{TconeD}
\mathcal{T}_{\A}(\x^\star) &= \operatorname{cone} \{ \z - \x^\star ~\ : ||\z||_\A \leq |\|x^\star||_\A \}\\
\label{NconeD}
\mathcal{N}_\A(\x^\star) &= \{ \vct{u}~:~ \langle \vct{u},\z \rangle \leq 0,~~\forall \z \in \mathcal{T}_{\A}(\x^\star)\} \\
\nonumber &= \{\vct{u}~:~ \langle \vct{u}, \x^\star \rangle = \gamma \|\x\|_{\A}~ \\
\nonumber ~&\mbox{and}~\|\vct{u}\|_{\A}^* \leq \gamma~\mbox{for some}~\gamma \geq 0\}
\end{align}

We note that, from \cite{venkat} (Prop. 2.1), $\hat{\x} = \x^\star$ (\ref{minAnorm}) is unique \emph{iff} 
\begin{equation}
\label{prop}
\mbox{null}(\mtx\Phi) \cap \mathcal{T}_{\A}(\vx^\star) = \{0\} 
\end{equation}
Hence, we require that the tangent cone at $\x^\star$ intersects the nullspace of $\mtx\Phi$ only at the origin, to guarantee exact recovery. 

Before we state the main recovery result from~\cite{venkat}, we define the \emph{Gaussian width} of a set:
\begin{definition}
Let $\mathbb{S}^{p-1}$ denote the unit sphere in $\R^p$.  The Gaussian width $\omega(S)$ of a set $S \in \mathbb{S}^{p-1}$ is 

\[
\omega(S) = \mathbb{E}_{\vct{g}} \left[ \sup_{\z \in S} \vct{g}^T\z \right]
\]

where $\vct{g} \sim \mathcal{N}(0,I)$ is an i.i.d. standard Gaussian vector.
\end{definition}
Gordon used the Gaussian width to provide bounds on the probability that a random subspace of a certain dimension misses a subset of the sphere~\cite{gordon}.  In~\cite{venkat}, these results are specialized to the case of atomic norm recovery.  In particular, we will make use of the following: 

\begin{proposition} \label{corl:width}[\cite{venkat}, Corollary 3.2]
Let $\mtx\Phi: \R^p \rightarrow \R^n$ be a random map with i.i.d. zero-mean Gaussian entries having variance $1/n$.  Further let $\Omega = T_\A(\x^*) \cap \mathbb{S}^{p-1}$ denote the spherical part of the tangent cone $T_\A(\x^\star)$.  Suppose that we have measurements $\vy = \mtx\Phi \vx^\star$, and we solve the convex program (\ref{minAnorm}). Then $\x^\star$ is the unique optimum of (\ref{minAnorm}) with high probability provided that
	\begin{equation*}
	n \geq \omega(\Omega)^2+\mathcal{O}(1).
	\end{equation*}
\end{proposition}

To complete our problem setup we will also restate Proposition 3.6 in \cite{venkat} :
\begin{proposition} ( \cite{venkat}, Proposition 3.6)
\label{prop3.6}
Let $C$ be any non-empty convex cone in $\R^p$, and let $\vct{g} \sim \N(0,I)$ be a Gaussian vector. Then:
\begin{equation}
\label{propVenkat}
\omega(C \cap \mathbb{S}^{p-1}) \leq \E_{\vct{g}}[\mathrm{dist}(\vct{g},C^*)]
\end{equation}
where $\mathrm{dist}( . , .)$ denotes the Euclidean distance between a point and a set, and $C^*$ is the dual cone of $C$ 
\end{proposition}

We can then square (\ref{propVenkat}) use Jensen's inequality to obtain
\begin{equation}
\label{jensen}
\omega(C \cap \mathbb{S}^{p-1})^{2} \leq \E_g[\mbox{dist}(\vct{g},C^*)^2]
\end{equation}
We note here that the dual cone of the tangent cone is the normal cone, and vice-versa.

Thus, to derive measurement bounds, we only need to calculate the square of the Gaussian width of the intersection of the tangent cone at $\vx^\star$ with respect to the atomic norm and the unit sphere. This value can be bounded by the distance of a Gaussian random vector to the normal cone at the same point, as implied by  (\ref{jensen}). In the next section, we derive bounds on this quantity.

\section{Gaussian Width of the Normal Cone for Unions of Subspaces}
\label{sec:result}

For generic subspaces $\K$, we have

\begin{align}
\notag
&\vct{c} \in \N_\A(\vx^\star) \iff 
\exists \gamma \geq 0 ~\ :  \langle \vct{c}, \vx^\star \rangle = \gamma \|\vx^\star\|_{\A}, \\
\label{eqn:normalGen}
& \|\bar{\vct{c}}^i\|=\gamma \mbox{ if } K_i \in \K^\star ,~\  \|\bar{\vct{c}}^i\|\leq \gamma \text{ if } K_i\not\in \K^\star.
\end{align}

We now prove the main result of this paper, a sufficient number of Gaussian measurements needed to recover a signal lying in a union of subspaces:
\[\]

\begin{theorem}
\label{mainTh}
To exactly recover a $k$-subspace sparse signal decomposed into $M$ subspaces in $\R^p$, 
\[
2 \frac{\sigma_p^2(\mtx{K}^\star)}{\sigma_p^2(\mk)}(\sqrt{2\log(M-k)} + \sqrt{B})^2 k + 2kB \kappa^2(\mk)
\]
 i.i.d. Gaussian measurements are sufficient.
\end{theorem}

To prove this result, we need two lemmas:
\[\]
\begin{lemma}\label{lemma:chisquos}
Let $q_1,\ldots,aq_L$ be $L$,  $\chi$-squared random variables with $d$-degrees of freedom.  Then \[
\E[\max_{1\leq i \leq L} q_i]\leq (\sqrt{2\log(L)} + \sqrt{d})^2.
\]
\end{lemma}
\[\]
We defer the proof to appendix \ref{appa}.
\[\]

\begin{lemma}\label{lemma:ballbound}
Suppose $\vv \in \R^p$ is supported on some set of groups $\K^\star \subset \K$. Then, 
\[
	\|\vv\| \leq \sqrt{|\K^\star|}~\ \sigma_p(\mtx{K}^\star)~\ \|\vv\|_{\A}^*\,.
\]
\end{lemma}
We defer the proof of this lemma to appendix \ref{appb}
\[\]

\begin{proof}[Proof of Theorem~\ref{mainTh}]

\emph{Intuition}:
Note that, from (\ref{jensen}), the Gaussian width of the intersection of the tangent cone at $\vx^\star$ with the unit sphere is bounded above by the expected euclidean distance between a random Gaussian vector and the normal cone at $\vx^\star$ (\ref{eqn:normalGen}). We can further bound this distance by the distance between a random Gaussian vector $\vct{g}$ and a particular vector $\vct{r} \in \N_\A(\vx^\star)$, as shown in (\ref{morebound}). We proceed to construct such a vector $\vct{r}$ and prove the result
\begin{equation}
\label{morebound}
\E_g[\mbox{dist}(\vct{g},C^*)^2] \leq \E_g[\mbox{dist}(\vct{g},\vct{r})^2] , ~\ \vct{r} \in \N_\A(\vx^\star)
\end{equation}
 
Now, since we assume $\mathrm{span}(\K) = \R^p$, we can write, for any vector $\vct\eta$, 
\[
\vct{\eta} = \mtx{K} \bar{\vct{\eta}} ~\  \bar{\vct\eta} \in \R^{\sum d_i}
\]

So, we have for $\vx^\star$,
\[
\vx^\star = \mk \bar{\vx^\star} = \mk^S\bar{\x^\star}^S + \mk^{S^c}\bar{\vx^\star}^{S^c}
\]
where $S \subset \{ 1,2, \ldots, \sum_i d_i \}$ is the indices of active coefficients of $\bar{\vx^\star}$. It is important to note that
 $\forall i \notin J, ~\ \vx^{\star i}$ is a sub vector of ${\vx^\star}^{S^c}$
 Also, note here that $\mk^{S^c}\bar{\vx^\star}^{S^c} = 0$.

Since the normal cone is nonempty, there exists a $\vv\in \mathcal{N}_{\mathcal{A}}(\vx^\star)$ with $\|\vv\|_{\mathcal{A}}^*=1$ and $\vv^i =0 ~\ \forall i \notin J$.  Since $\vv$ is in the normal cone, it will also satisfy $\langle \vv, \vx^\star \rangle = \|\vx^\star\|_\mathcal{A}$.  We will use this $\vv$ in our bound below.

Suppose $\vw \sim \N(0,\mtx{I_p})$ is a vector with i.i.d. Gaussian entries.  We then have

\begin{align*}
\vct{w} &= \mtx K \bar{\vct{w}} \\
&=  \mtx{K^S} \bar{\vct{w}}^S ~\ + ~\ \mtx{K^{S^c}} \bar{\vct{w}}^{S^c}
\end{align*}
Let $t(\vw) = \max_{i \notin J} \|w^i\|$. 

since $w = \sum_{i = i}^M K_i \bw^i$ , we have $\bw^i = \mk_i^T \left( \mk \mk^T \right)^{-1} \vw $, giving us $\bw^i  \sim \N(0, \mk_i^T \left( \mk \mk^T \right)^{-2} \mk_i)$. This means that 
\begin{equation}
\label{wchi}
\| \bw^i \|^2 \sim \| \left( \mk \mk^T \right)^{-2} \|  \chi^2_{d_i} 
\end{equation}
So, $\| \bw^i \|^2$ is a scaled $\chi^2$ random variable with $d_i$ degrees of freedom.  The scaling factor is merely $\sigma_p^{-2}(\mk)$. 

Let us now construct a vector $\vct{r} \in \N_\A(\vx^\star)$. We can write, as for $\vw$
\begin{align*}
\vct{r} &=\mtx{K^S} \bar{\vr}^S ~\ + ~\ \mtx{K^{S^c}} \bar{\vr}^{S^c}
\end{align*}
Now let $\bar\vr^S = t(\vw) \bar\vv^S$, and $\bar\vr^{S^c} = \bar\vw^{S^c}$

From  \eqref{eqn:normalGen}, and from our definition of $t(\vw)$, we have $\vr \in \N_\A(x^\star)$. Referring to  (\ref{jensen}), we now consider the expected squared distance between $\N_\A(\vx^\star)$ and $\vw$:

\begin{align*}
&\E[\mbox{dist}(\vw,C^*)^2]\\ 
&\leq \E[|| \vr - \vw ||^2] \\  
&= \E \left[ \left\| \mk^S \br^S +  \mk^{S^c} \br^{S^c}  - \mk^S \bw^S  + \mk^{S^c} \bw^{S^c} \right\|^2 \right] \\
&\stackrel{(i)}= \E \left[ \left\| \mk^S \br^S - \mk^S \bw^S \right\|^2 + \left\| \mk^{S^c} \br^{S^c} -  \mk^{S^c} \bw^{S^c} \right\|^2\right] \\
&= \E \left[ \left\|\mk^S \br^S - \mk^S \bw^S \right\|^2 \right] \\
&\stackrel{(ii)}= 2 \E \left[ \left\|\mk^S \br^S \right\|^2 \right] + 2 \E \left[\left\| \mk^S \bw^S  \right\|^2\right]\\
&= 2\E \left[ \left\| \mk^S t(\vw) \bv^S \right\|^2 \right] +2  \E \left[\left\|  \mk^S \bw^S  \right\|^2\right] \\
&\stackrel{(iii)}= 2 \E [t(\vw)^2] \left\| \mk^S \bv^S   \right\|^2 + 2 \E \left[  \left\| \mk^S \bw^S \right\|^2 \right] \\
&= 2 \E [t(\vw)^2] \left\| \mk \bv   \right\|^2 +  2 \E \left[  \left\| \mk^S \bw^S \right\|^2 \right] \\
&\stackrel{(iv)}= 2 \E [t(\vw)^2] \left\| \vv \right\|^2 +2 \E \left[  \left\| \mk^S \bw^S \right\|^2 \right] \\
&\stackrel{(v)}\leq2  k \left( \frac{\sigma_p(\mtx{K^\star})}{\sigma_p(\mk)} \right)^2 (\sqrt{2\log(M-k)} + \sqrt{B})^2 +2 \E \left[  \left\| \mk^S \bw^S \right\|^2 \right] \\
&\stackrel{(vi)}\leq 2k  \left( \frac{\sigma_p(\mtx{K^\star})}{\sigma_p(\mk)} \right)^2 (\sqrt{2\log(M-k)} + \sqrt{B})^2 + 2kB \kappa^2(\mk)
\end{align*}

Where (i) trivially follows because the indices in $S$ and $S^c$ are disjoint, (ii) follows from the result $\|a - b\|^2 \leq 2 (\|a\|^2 + \|b\|^2)$ (iii) follows from the fact that $\vv$ is deterministic. (iv) follows from the fact that $\bv$ is only supported on $S$ (v) follows from Lemma \ref{lemma:chisquos},  Lemma \ref{lemma:ballbound} and (\ref{wchi}). Finally, (vi) follows from bounding the last term as shown in appendix \ref{appc},  and noting that $|S| \leq kB$.
\end{proof}

\subsection{Remarks}
\begin{enumerate}
\item The most important thing to note from our result is that we pay no extra penalty in terms of the number of measurements needed when the subspaces overlap. Hence, we term our result ``universal". 

\item The $k B$ term in the bound is an upper-bound on the signal sparsity. In the case of highly overlapping subspaces, this value may be much larger than the signal sparsity, but such cases seldom arise in real-world applications. If the subspace dimensions are vastly different, then it is pessimistic to bound the quantity with the maximum dimension $B$, but this yields a simple expression for the measurements needed.  It is of course possible to obtain tighter bounds using the techniques in our work for cases where the groups are of varying sizes. 

\item It can be seen from Theorem \ref{mainTh} that the number of measurements is linear in $k$ and $B$. Hence, the number of measurements that are sufficient for signal recovery grows linearly with the number of active subspaces in the signal, and also the maximum subspace dimension. This can be seen analogous to the linear dependence of the lasso bound on the sparsity $s$ of the signal.

\item We note that although we pay no extra price to measure the signal when there is significant overlap between subspaces, there is an additional cost in the recovery process of the signal, in that the subspaces need to first be separated by replication of the coefficients \cite{jacob}, or resort to a primal-dual method to solve the problem \cite{mosci}.

\item In the bound we get, 
 $\sigma_p(\mk)$ captures the price we pay if the subspaces are non perpendicular. Indeed, if subspaces are nearly aligned with each other, then it becomes nearly impossible to distinguish between them. This is reflected by the fact that 
 $\sigma_p(\mk) \rightarrow 0$ as the subspaces become more aligned. Fig. \ref{numsigs} shows this phenomenon for the case of two subspaces of one dimension each that become more and more aligned with each other.  Similarly, in the second term, the condition number of the matrix $\mtx{K}$ , $\kappa(\mtx{K})$ is determined by the angle between subspaces. The more the subspaces are close to each other, the higher the condition number, and subsequently the more measurements we need.  This can be seen from Fig. \ref{kappa}.
\end{enumerate}

\begin{figure}[!h]
\begin{center}
\subfigure[$\sigma_p(\mk)$ as the angle between 2 subspaces is varied]{
\includegraphics[scale = 0.35]{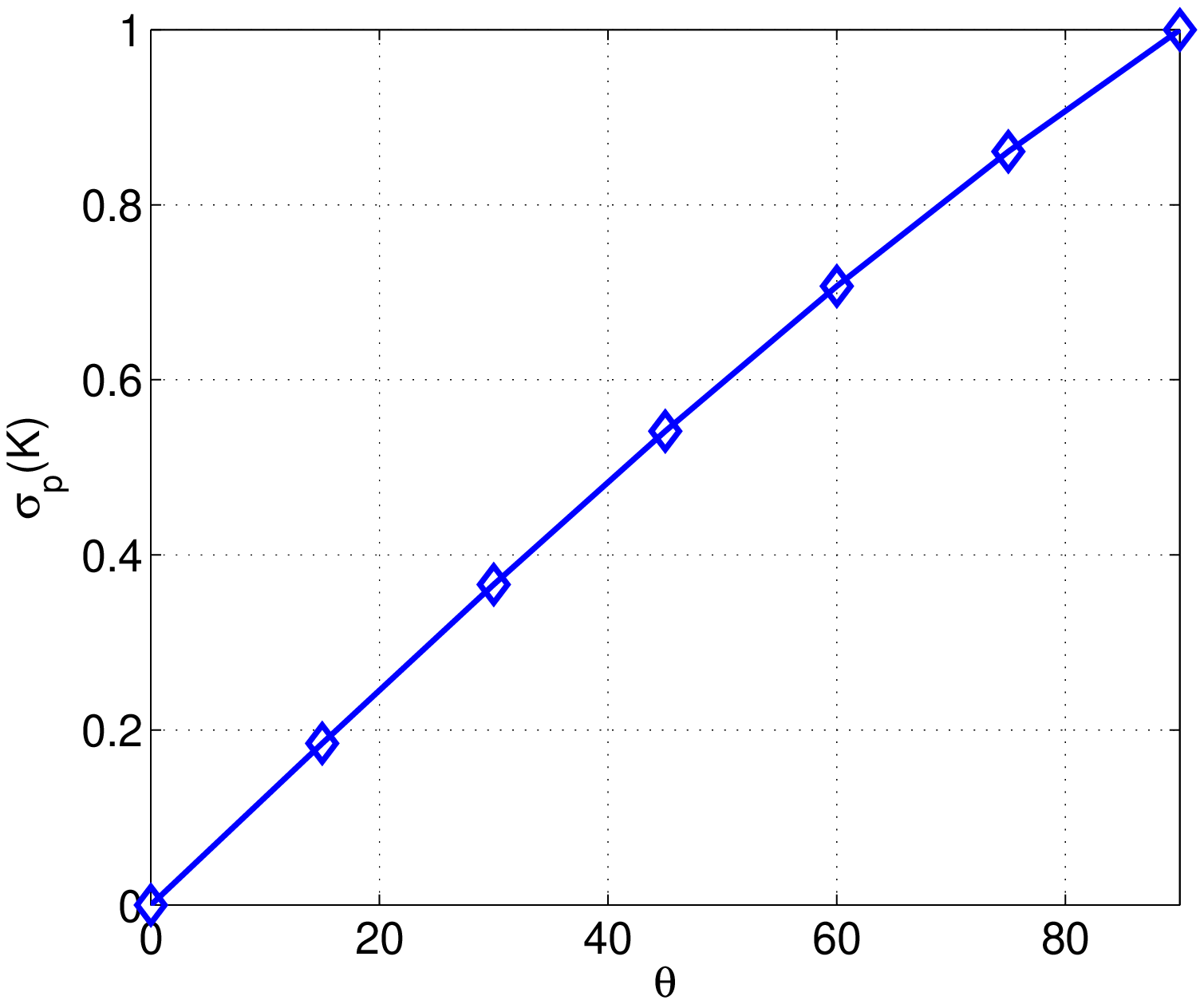}
\label{numsigs}}
\subfigure[$\kappa(\mk)$ as the angle between 2 subspaces is varied]{
\includegraphics[scale = 0.4]{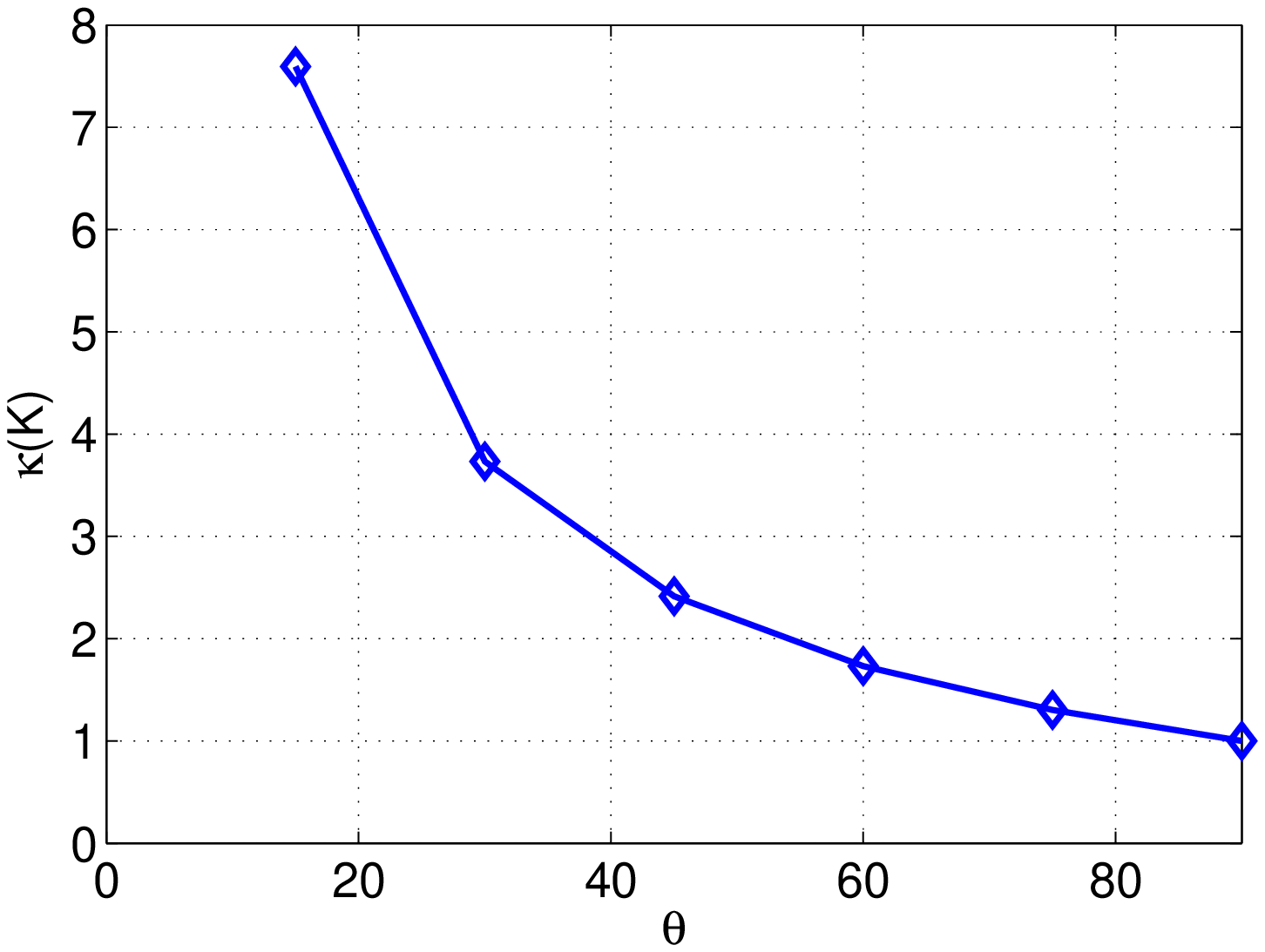}
\label{kappa}}
\caption{As the two subspaces become more separated ($\theta \rightarrow 90^\circ $), both the quantities approach 1. As ($\theta \rightarrow 0^\circ $), $\sigma_p(\mk) \rightarrow 0$ and $\kappa(\mk) \rightarrow \infty$, indicating that it becomes impossible to distinguish between active subspaces}.
\label{fig:constants}
\end{center}
\end{figure}



\subsection{Perpendicular Subspaces}
When the subspaces are perpendicular to each other, as defined in (\ref{eqperp}), we can make the bound we obtained much tighter. To see this, note that when the subspaces are perpendicular,  equation (ii) in the proof of Theorem \ref{mainTh} can be replaced by 
\[
\E[\|\mk^S\br^S\|^2] + \E[\|\mk^S\bw^S\|^2]
\]
This follows since, in the case of perpendicular subspaces, $\br^S$ is independent of $\bw^S$. Also, since the subspaces are perpendicular, we have
\[
\sigma_p(\mk^\star) ~\ = ~\ \sigma_p(\mk) ~\ = ~\ \kappa(\mk) ~\ = ~\ 1
\]
Substituting the values in the bound we get from Theorem \ref{mainTh}, we have
\begin{equation}
\label{perpresult}
\E[\mbox{dist}(\vw,C^*)] \leq k(\sqrt{2\log(M-k)} + \sqrt{B})^2 + kB
\end{equation}
for perpendicular subspaces. This is a much smaller quantity than the general result we obtained in Theorem \ref{mainTh}, underscoring the fact that recovery becomes easier as the subspaces get more and more separated. In the next section, we show that the group lasso with overlapping groups (also called the latent group lasso) is a special case of recovery when subspaces are perpendicular.

\section{The Group Lasso with Overlapping Groups}
\label{sec:glassoresult}
The group lasso with overlapping groups, \cite{jacob, latent, percival} can be formulated as an atomic norm minimization problem. In the group lasso problem, we are given a set of groups $\G = \{G_1, G_2, \ldots, G_M\}$, $G_i \subset \{1,2, \ldots p\}$ and we wish to recover a k-group sparse vector $\vx^\star$ from compressive measurements. In this case, we can define the subspaces $K_i \in \K$ as follows:

\begin{equation}
\label{extglasso}
\forall i \in \{1,2, \ldots, M\}, K_i = \mtx{I}^{G_i}
\end{equation}
where $\mtx{I}^{G_i}$ is the sub matrix of the identity matrix, consisting of columns indexed by the group $G_i$.

We now show that minimizing the atomic norm under the atomic set arising out of these subspaces yields the group lasso with overlapping groups. Note that, under the definition of the subspaces as in (\ref{extglasso}), and referring to (\ref{aset}), we have that $A_i$ is merely the unit sphere restricted to the dimensions indexed by group $G_i$.

\begin{lemma}
\label{lemogl}
Suppose the atomic set is given as in (\ref{aset}). Let $K_i = \mtx{I}^{G_i} ~\ \forall i = \{1,2, \ldots M\}$, where $G_i \subset \{1,2, \ldots p\}$. Then,  
\[
||\vx||_{\A} = \Omega_{overlap}^\G (\vx)
\]
where $\Omega_{overlap}^\G (\vx)$ is the overlapping group lasso norm defined in \cite{jacob}.
\end{lemma}
\begin{proof}
In (\ref{anormdef}), we can substitute $\vv_G = c_G  \va$, giving us $c_G = |c_G| \cdot ||\va|| = || c_G \va|| = \|\vv_G\|$. Hence, 
\begin{align*}
||\vx||_{\A} &= \inf \left\{ \sum_{\va \in \A} c_a : \vx =  \sum_{\va \in \A} c_a\va ~\  c_a \geq 0 ~\  \forall \va \in \A \right\} \\
&= \inf \left\{ \sum_{G \in \G} ||\vv_G|| ~\ : \vx = \sum_{G \in \G} \vv_G \right\} \\
&= \Omega_{overlap}^\G(\vx)  
\end{align*}
\end{proof}

\begin{corollary}
\label{corrgl}
Under the atomic set defined in (\ref{aset}), when $K_i = \mtx{I}^{G_i} ~\ \forall i = \{1,2, \ldots M\}$,
\[
||\vx||_{\A} = \sum_{G \in \G} ||\vx_G||
\]
\end{corollary}
\begin{proof}
$\Omega_{overlap}^\G (\vx) = \sum_{G \in \G} ||\vx_G||$ in the non overlapping case. 
\end{proof}
\[\]
Thus, (\ref{minAnorm}) yields:
\begin{equation}
\label{minAnormov}
\hat{\vx} = \underset{\x \in \R^p}{\operatorname{argmin}} ~\ \Omega_{overlap}^\G(\x)  ~\ \textbf{s.t. } \y = \mtx\Phi \x
\end{equation}
which can be solved using  \cite{jacob}.

It is not hard to see that, in the case of disjoint groups, 
\begin{align}
\label{disjointcone}
\mathcal{N}_\A(\vx^\star) &= \{ \vz \in \mathbb{R}^p : \vz_i = \gamma \frac{(\vx^\star)_i}{||\vx^\star_G||}  ~\ \forall G \in \G^\star , \\  \notag &~\   ||\vz_G|| \leq \gamma ~\ \forall G ~\notin ~\G^\star , \gamma  \geq 0 \} \end{align}
However, in the case of overlapping groups, as in the case of generic subspaces, no such closed form exists.  

Under the group sparsity model defined above, it is not hard to see that $\sigma_p(\mtx{K^\star}) = \kappa(\mk) = 1$. Also, note that the subspaces are perpendicular to each other since the basis vectors of each subspace is aligned with one or more of the coordinate axes in $\R^p$. As a consequence, we obtain the following result:
\[\]
\begin{theorem}
\label{mainThglasso}
To exactly recover a $k$-group sparse signal decomposed into $M$ groups in $\R^p$, the following is a sufficient number of Gaussian measurements needed:
\[
(\sqrt{2\log(M-k)} + \sqrt{B})^2 k + kB
\] 
\end{theorem}

\subsection{Remark: Comparison with the lasso}
 We compare the group lasso bound we obtain to the standard lasso measurement bound: 
\begin{equation}
\label{venkatbound}
(2s+1)\log(p-s)
\end{equation}
The bound we obtain in Theorem \ref{mainThglasso} can be upper bounded by
\begin{equation}
\label{ourbound}
2k  \max\{2\log(M), B\} + kB
\end{equation}

Noting that $s \leq kB$ with equality when the groups do not overlap. In this case, (\ref{ourbound}) evaluates to
\begin{align*} 
&~\  ~\ \frac{2s}{B}  \max\{2\log(M), B\} + s \\
&= (2s + 1)\frac{\max\{2\log(M), B\}}{B}
\end{align*}
which is smaller than the lasso bound (\ref{venkatbound}) by a factor of roughly $\frac{\log(M)}{B\log(p)}$. So, in most cases, our bound shows that the we can perform better than the conventional lasso by exploiting the additional group structured information that is available.

\section{Approximately Sparse Signals}
\label{sec:approx}
The result we proved in Theorem \ref{mainTh} apply to exact reconstruction of  $k-$ subspace sparse signals. In many cases however, the signals are not exactly subspace (or group) sparse, but approximately so. Specifically, in a context we are especially interested in, the ordered wavelet coefficients of natural images exponentially decay to zero, and hence can be modeled as \emph{approximately} sparse, the ``sparsity" meaning that only few coefficients are of significant magnitude. 

In such cases, we can model the approximately sparse signal $\vct{f}^\star$ as 
\[
\vct{f}^\star = \vx^\star + \vct{h}^\star
\]
where $\vx^\star$ is a $k-$ subspace sparse approximation of $\vct{f}^\star$, retaining the $k$ subspaces having largest norm and $\vct{h}^\star$ corresponds to the remaining coefficients that are small in magnitude. Clearly, we can bound $\|\vct{h}^\star\|$ above by some constant, say $c_h$.

Now, measuring the approximately sparse signal $\vct{f}^\star$ using a Gaussian measurement matrix amounts to
\begin{align*}
\mtx\Phi \vct{f}^\star &= \Phi \vx^\star + \Phi \vct{h}^\star \\
              &= \mtx\Phi \vx^\star + \vct\theta
\end{align*}
Since the norm of $\mtx\Phi$ is bounded, we can write 
\[
\|\vct\theta\| \leq \delta
\]
The results we have obtained thus far can be easily extended to the case where we obtain such bounded noisy observations. In the noisy case, we observe 
\[
\vy = \mtx\Phi \vx^\star + \vct\theta , ~\  \|\vct\theta\| \leq \delta
\]
We then solve the atomic norm minimization problem, with a relaxed constraint to take into account the bounded noise:
\begin{equation}
\label{minAnormNoise}
\hat{\vx} = \underset{\vx \in \R^p}{\operatorname{argmin}} ||\vx||_\A  ~\ \textbf{s.t. } \|\vy - \mtx\Phi \vx\| \leq \delta
\end{equation}
We restate corollary $3.3$ from \cite{venkat}:

\begin{proposition} \label{corl:widthnoisy}[\cite{venkat}, Corollary 3.3]
Let $\mtx\Phi: \R^p \rightarrow \R^n$ be a random map with i.i.d. zero-mean Gaussian entries having variance $1/n$.  Further let $\Omega = T_\A(\vx^*) \cap \mathbb{S}^{p-1}$ denote the spherical part of the tangent cone $T_\A(\vx^\star)$.  Suppose that we have measurements $\vy = \mtx\Phi \vx^\star + \vct\theta$, and $\|\vct\theta\| \leq \delta$. Suppose we solve the convex program (\ref{minAnormNoise}). Let $\hat{\vx}$ denote the optimum of (\ref{minAnormNoise}). Also, suppose $\| \mtx\Phi \vz \| \geq \epsilon \|\vz\| ~\ \forall \vz \in T_\A(\vx^\star)$. Then $\|\vx^\star - \hat{\vx}\| \leq \frac{2\delta}{\epsilon}$  with high probability provided that
	\begin{equation*}
	n \geq \frac{\omega(\Omega)^2}{(1-\epsilon)^2}+\mathcal{O}(1).
	\end{equation*}
\end{proposition}

Substituting the result of Theorem \ref{mainTh} in Proposition \ref{corl:widthnoisy}, we have the following corollary yielding a sufficient condition to accurately recover a signal when the measurements are corrupted with bounded noise:

\begin{corollary}
\label{corr:noisy}
Suppose we wish to recover a  signal that lies in $k$ out of $M$ arbitrarily defined subspaces, such that the maximum subspace dimension is $B$. Let the set of active subspaces be denoted by $\K^\star$. Let $\hat{\vx}$ be the optimum of the convex program (\ref{minAnormNoise}). To have $\|\hat{\vx} - \vx^\star\| \leq \frac{2\delta}{\epsilon}$ with high probability, 
\[
 \frac{2\frac{\sigma_p^2(\mtx{K^\star})}{\sigma_p^2(\mk)}(\sqrt{2\log(M-k)} + \sqrt{B})^2 k + 2kB \kappa^2(\mk)}{(1-\epsilon)^2}
 \]
 i.i.d. Gaussian measurements are sufficient.
 \end{corollary}

Note that we merely need to set $\sigma_p(\mtx{K^\star}) =  \kappa(\mk) = 1$, and remove the `2' from both terms  to obtain the corresponding result for the latent group lasso:
\[
 \frac{(\sqrt{2\log(M-k)} + \sqrt{B})^2 k + kB }{(1-\epsilon)^2}
 \]

\section{Compressive Imaging with Group Sparsity}
\label{sec:icip}
We consider the compressive imaging problem, that is to recover an image from a small number of random measurements. Here ``small" is used relative to the ambient dimension of the image. 
The standard lasso \cite{tibshirani} formulation is given by 
\begin{equation}
\label{eq:lasso}
\hat{\vx} = \arg \min_{\vx} \frac12 \|\vy - \mtx\Phi \vx \|^2 + \lambda \|\vx\|_1
\end{equation}

The $\ell_1$ norm acts as a surrogate for the sparsity of the signal. The lasso aims to recover a signal that is sparse, by setting most coefficients of $\vx$ to be zero.  For the exact recovery case, the lasso problem is equivalent to the Basis Pursuit \cite{BP}
\begin{equation}
\label{eq:bp}
\hat{\vx}  = \arg \min_{\vx} \|\vx\|_1 ~\ \text{ s.t. } \vy = \mtx\Phi \vx
\end{equation}

The lasso penalty reflects the fact that the wavelet coefficients are approximately sparse, but in reality not all patterns of sparsity are equally plausible/probable. For example, Fig (\ref{camW}) shows the DWT coefficients of the barbara image, and Fig. (\ref{randcamW}) shows the same coefficients, but randomly scrambled. Clearly, the $\ell_1$ norm of both sets of coefficients will be the same. This shows that the lasso penalty in itself is invariant to any structure present in the sparse coefficients. 
\begin{figure}[!htp]
\begin{center}
\subfigure[Original image]
{\includegraphics[scale = 0.35]{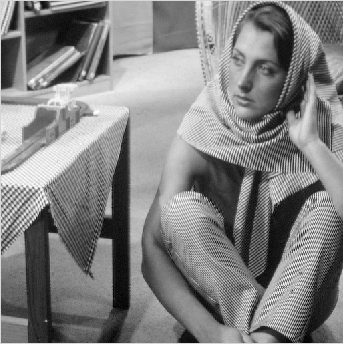}
\label{camo}}
\subfigure[3 -stage DWT of the barbara image]
{\includegraphics[scale = 0.35]{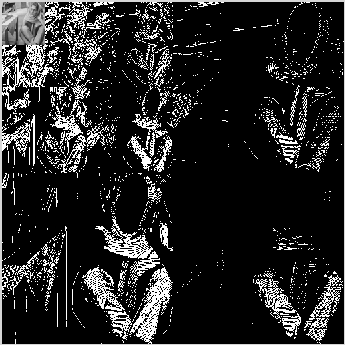}
\label{camW}}
\subfigure[Coefficients of the DWT of the Barbara image, randomized]
{\includegraphics[scale = 0.35]{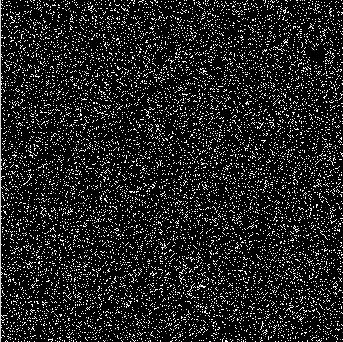}
\label{randcamW}}
\caption{The $\ell_1$ norms of both (a) and (b) are exactly equal, since they do not take structure into account} 
\label{camDWT}
\end{center}
\end{figure}

To model this structure that is inherently present between wavelet transform coefficients of images, \cite{romberg, crouse98, duarte08} propose making use of graphical models such as Hidden Markov Trees. HMT's, while providing good performance in image denoising applications (where $\mtx\Phi = \mtx{I}$) in (\ref{eq:lasso}), cannot provide acceptable reconstruction for other, more general inverse problems. This is because the presence of a (non identity) sensing matrix $\mtx\Phi$ (randomly) mixes up the coefficients for every measurement $\vy_i$ obtained. 

To overcome this mixing between the coefficients, many alternatives have been proposed. \cite{schniter10} propose using a version of loopy belief propagation to solve the recovery problem. The authors in \cite{modelbased, model09 } generalize the notion of restricted isometry properties to signals that lie in unions of subspaces, and use a modified version of CoSAMP \cite{cosamp} to solve the inverse problem. Greedy and/or suboptimal iterative reconstruction schemes are used in \cite{lado, duarte08}. Finally, the authors in \cite{som11} propose modeling the coefficients using an HMT, and using the Approximate Message Passing algorithm \cite{montanari} to solve the compressed sensing problem. 

All the works mentioned above sacrifice the recovery guarantees and the easy analysis that convex optimization algorithms provide, for the sake of modeling the dependencies between DWT coefficients. This motivates out problem: can we on the one hand model the dependencies among wavelet transform coefficients, while at the same time propose to solve a convex optimization problem similar to (\ref{eq:lasso})? 

To this end, we model the parent-child coefficients into groups. Parent-child pairs of wavelet transform coefficients across scales and at similar locations tend to be simultaneously high or low. Hence, we can take advantage of this dependency and use group lasso methods to recover the image. Fig \ref{wtree2} shows a representative example. 

\begin{figure}[htp]
\begin{center}
\includegraphics[scale =0.7]{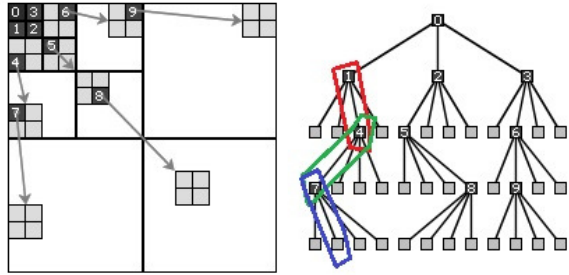}
\caption{Quadtree corresponding to the 2-d DWT. At each scale, parent coefficients can be grouped with child coefficients.}
\label{wtree2}
\end{center}
\end{figure}

We wish to recover the non zero coefficients lying on the wavelet tree shown in Fig \ref{wtree2}. When coefficients are modeled into groups, one can use the group lasso \cite{yuanlin} to recover the coefficients
\begin{equation}
\label{eq:glassoyuan}
\hat{\vx} = \arg \min_{\vx} \frac12 \|\vy - \mtx\Phi \vx \|^2 + \lambda \sum_{i = 1}^M\|\vx_{G_i}\|
\end{equation}

where $\vx_{G_i}$ is the vector $\vx$ whose coefficients not indexed by group $G_i$ are set to zero. The group lasso as shown in (\ref{eq:glassoyuan}) suffers from a drawback however. It was recently argued in \cite{jacob} that the sparsity pattern recovered by the group lasso can be expressed as a complement of a union of groups. One look at Fig. \ref{wtree2} tells us that we are interested in the recovery of sparsity patterns that can be expressed as a union of (overlapping) groups. To this end, the authors in \cite{jacob} propose the latent group lasso \cite{latent, percival}
\begin{equation}
\label{eq:glasso}
\hat{\vx} = \arg \min_{\vx} \frac12 \|\vy - \mtx\Phi \vx \|^2 + \lambda \Omega^{\G}_{overlap}(\vx)
\end{equation}

where $\Omega^{\G}_{overlap}(\vx)$ is the latent group lasso norm.

The latent group lasso lends itself well to the sort of problems we are concerned about in this paper. For a thorough analysis of the various properties of the latent group lasso penalty, we refer the interested reader to \cite{latent}. In the next section, we show how the latent group lasso can be effectively used to recover images, when we model the DWT coefficients to lie in parent-child groups. 

\section{Experiments}
\label{sec:expts}
In this section we aim to show two things:
\begin{enumerate}
\item The bound we derived in Theorems \ref{mainTh} and \ref{mainThglasso} holds for a wide variety of cases, and is invariant to the grouping observed in the signal
\item By modeling the DWT coefficients of images (and 1D signals of course) into parent-child groups, we can recover the signal efficiently and exactly/robustly
\end{enumerate}

Also, henceforth we refer to the latent group lasso method as Glasso.

\subsection{Sampling Bounds for Subspace Sparse Signals}
We extensively tested our method against the standard lasso procedure. In the case where the groups overlap, we use the replication method outlined in \cite{jacob}, to reduce the optimization problem to that of non overlapping groups. \vspace{-1mm}
We compare the number of measurements needed for our method with that needed for the lasso. For the lasso, it would be instructive to keep in mind the bound derived in \cite{venkat} , \emph{viz.} $(2s+1)\log(p-s)$. In the case of non overlapping groups, the bound evaluates to $(2kB + 1)\log(kM - kB)$. We generate length $p =2000$ signals, made up of $M = 100$ non-overlapping groups of size $B = 20$. We set $k = 5$ groups to be ``active", and the values within the groups are drawn from a uniform $[0,1]$ distribution. The active groups are assigned uniformly at random. The sparsity of the signal will thus be $s = 100$

We use SpaRSA \cite{sparsa} for the lasso and the group lasso with overlap, learning $\lambda$ over a grid. Fig. \ref{fig:complasso} displays the mean reconstruction error $||\hat{\vx} - \vx^*||_2^2 / p$ as a function of the number of random measurements taken. The errors have been averaged over 100 tests, and each time a new random signal was generated with the above mentioned parameters. 

From the parameters considered, we conclude that $\approx 380$ measurements are sufficient to recover the signal. When we have 380 measurements, the lasso does not recover the signal exactly, as seen in Fig \ref{fig:complasso}, but the latent group lasso does. 

\begin{figure}[!h]
\centering
\includegraphics[scale = 0.3]{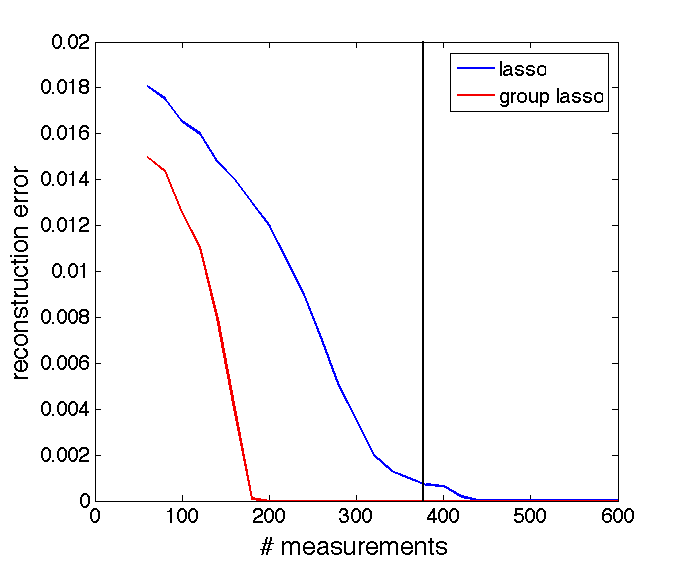}
\caption{The group lasso (red) compared with the lasso (blue). The vertical line indicates our bound. Note that our bound (380) predicts exact recovery of the signal, while at the same value, the lasso does not recover the signal}
\label{fig:complasso}
\end{figure}

To show that the bound we compute holds regardless of the complexity of groupings, we consider the following scenario: Suppose we have $M = 100$ groups, each of size $B = 40$. $k = 5$ of those groups are active, and the values within each group are assigned from a uniform $[-1,1]$ distribution. We arrange these groups in three configurations:

\begin{enumerate}
\item The groups do not overlap, yielding a signal of length $p = 4000$, and signal sparsity $s = 200$.
\item A partial overlapping scenario, where apart from the first and last group, every group has $20$ elements in common with a group above it, and $20$ common with the group below, giving $p = 2020$, $s \in [120, ~\ 200]$ depending on which of the 100 groups are active.
\item A random overlap case where the first 50 groups are non overlapping and the remaining 50 are assigned uniformly at random from the existing $p = 2000$ indices. $s \leq 200$ in this case.
\end{enumerate}
The scenarios we consider are depicted in Fig. \ref{fig:scenes}. In each of the cases, we compute the bound to be $\approx 630$.  The bound becomes looser as the complexity of the groupings increases. This, as argued before, is a result of the bound for the signal sparsity becoming looser.

\begin{figure}[!h]
\centering
\includegraphics[scale = 0.75]{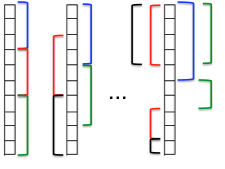}
\caption{Types of groupings considered. Each set of coefficients encompassed by one color belongs to one group.}
\label{fig:scenes}
\end{figure}

We can see from Fig. \ref{fig:scenes_resultsgl} that our group lasso bound $(\approx 630)$ holds for all cases. For the sake of comparison, we considered the lasso performance on the signals in cases (1) - (3) as well, and these are plotted in Fig. \ref{fig:scenes_resultsl}.  From the values of $p$ and $s$ computed for the three cases, we have the corresponding bounds for the lasso \cite{venkat} to be 3305 for the no overlap case (1), [1819, 3010] for the partial overlap case (2) and $3000$ for case (3). 

\begin{figure}[!h]
\begin{center}
\subfigure[performance of the group lasso on cases considered in Figure \ref{fig:scenes}. Note that our bound evaluates to 630, clearly sufficient measurements to recover the signal in all cases.]{
\includegraphics[scale = 0.45]{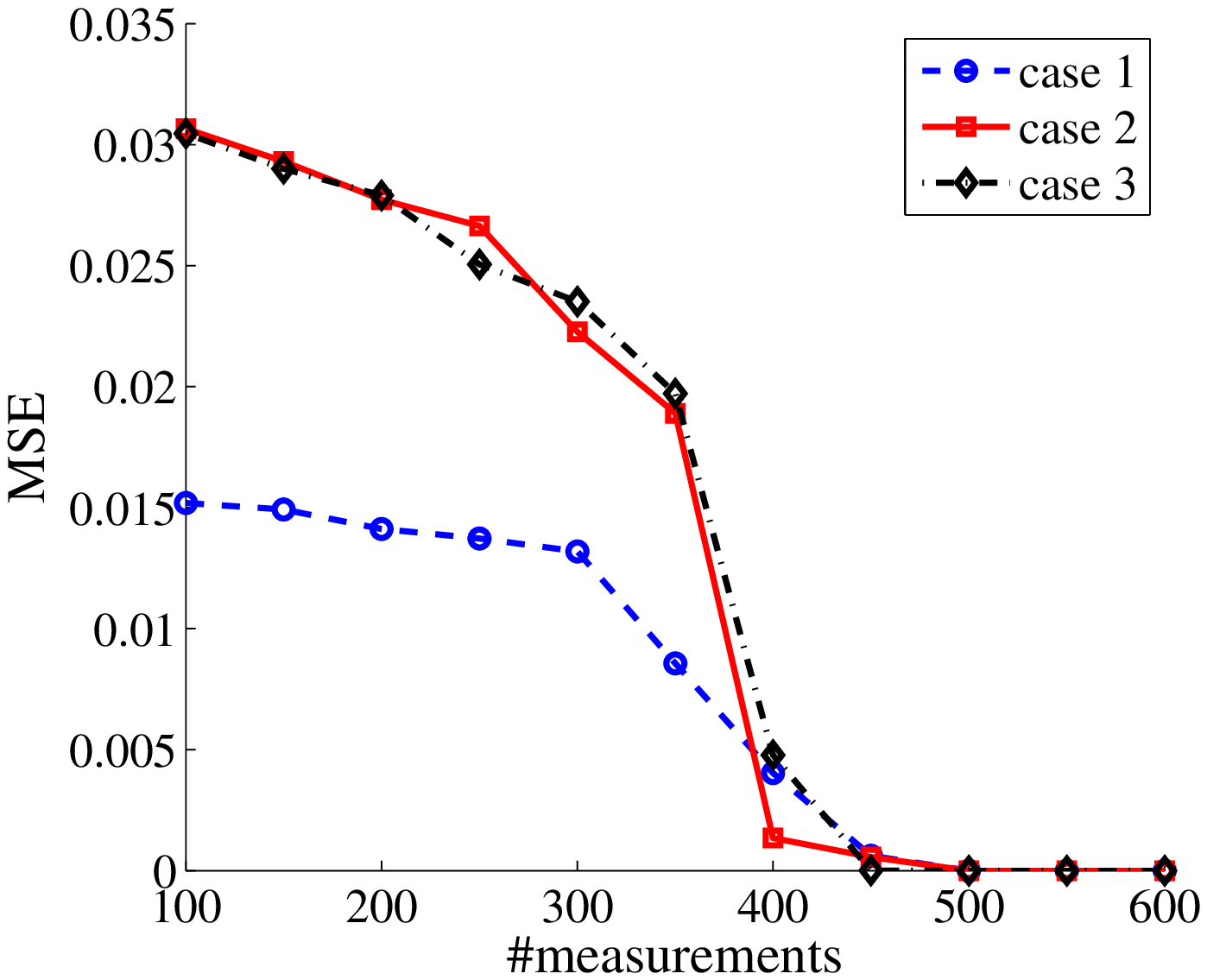}
\label{fig:scenes_resultsgl}}
\subfigure[performance of the lasso on cases considered in Figure \ref{fig:scenes}.]{
\includegraphics[scale = 0.45]{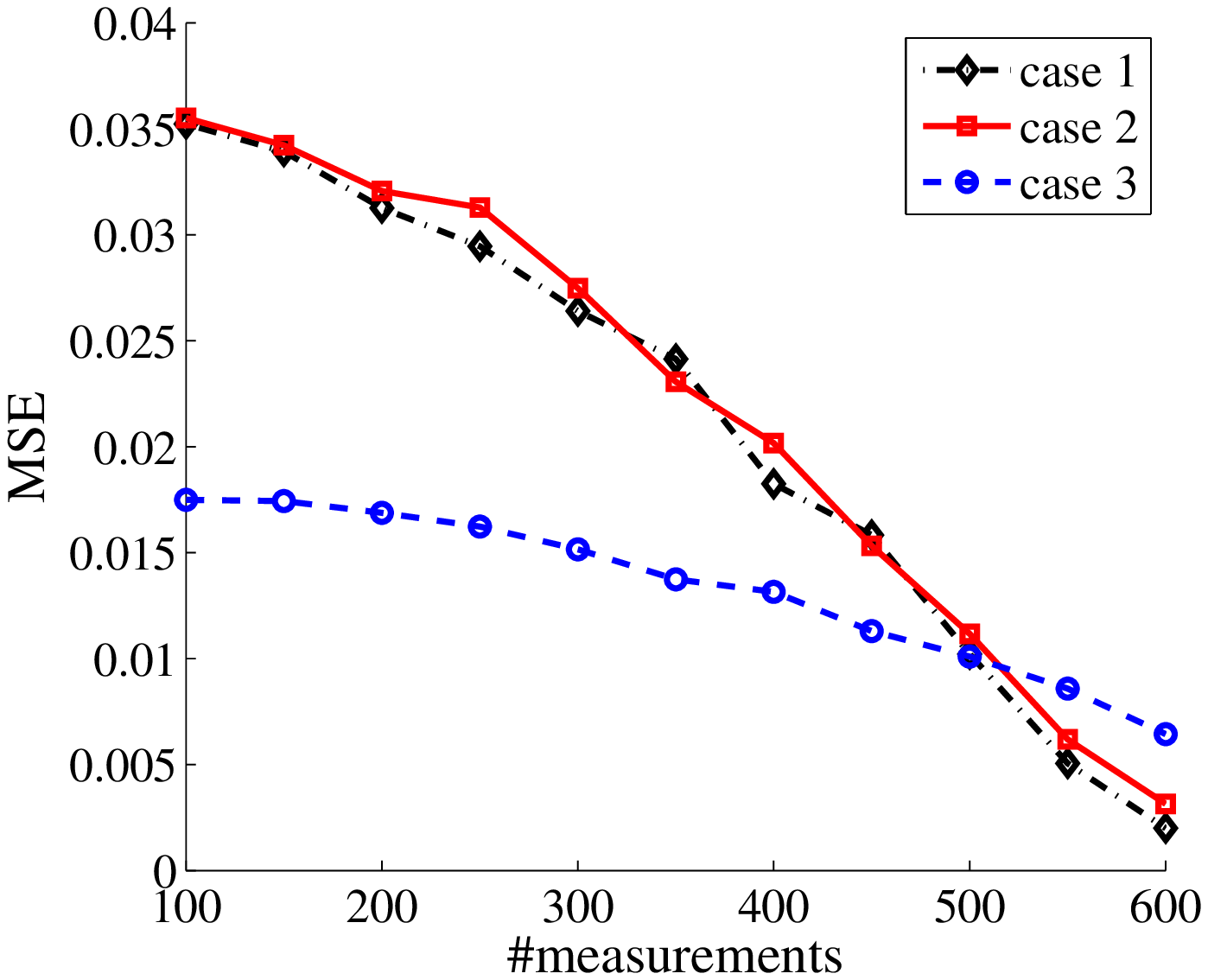}
\label{fig:scenes_resultsl}}
\caption{(Best seen in color) Performance on various grouping schemes.  The group lasso outperforms the lasso in all cases}
\end{center}
\end{figure}

We consider exact recovery of the wavelet transform coefficients of the ``blocks" signal (Fig. \ref{fig:blocks}). We group the wavelet transform coefficients into parent child pairs as outlined in Section \ref{sec:icip}.  In this case, for a $p = 16384$ length signal, we have $M = 16382$ groups, and the maximum group size is $B = 2$. We use the Haar wavelet bases to decompose the image. Fig. \ref{fig:recons} shows the reconstruction obtained from $1690$ measurements, corresponding to the bound computed for $k = 47$. Of course, we can compute $k$ since we have the original signal with us.  We see that our bound yields a sufficient number of measurements for exact recovery.

\begin{figure}[!h]
\begin{center}
\subfigure[original signal]{
\includegraphics[scale = 0.32]{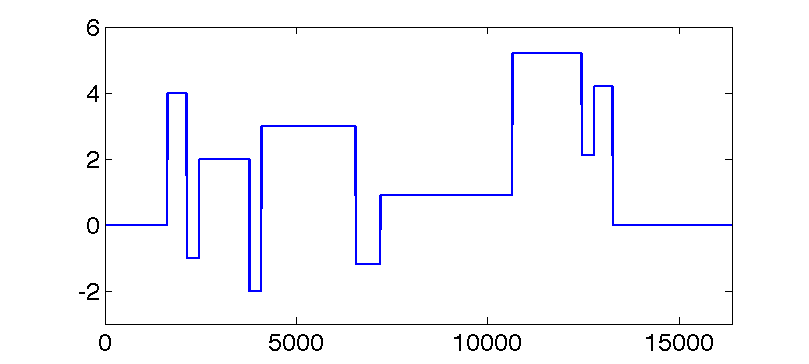}
\label{fig:blocks}}
\subfigure[reconstruction]{
\includegraphics[scale = 0.32]{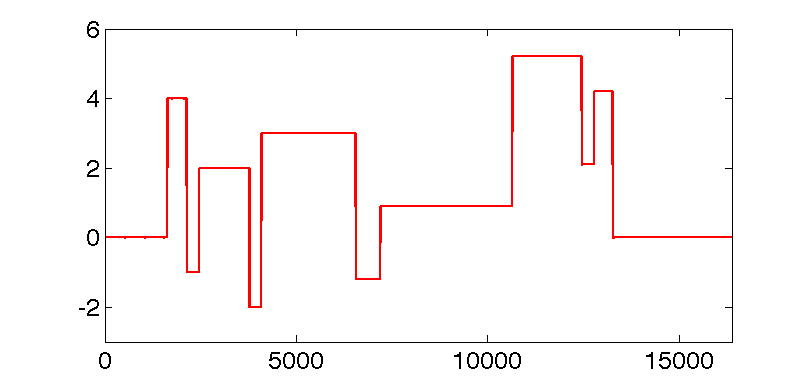}
\label{fig:recons}}
\caption{Exact reconstruction of a length 16384 signal from 1690 measurements in the wavelet domain}
\end{center}
\end{figure}

Our final experiment outlines the relationship between the number of measurements taken and the size of the problem. We generated test signals that were group sparse, with each active group having coefficients selected randomly from a uniform $\mathcal{U}[-1,1]$ distribution. We fix the group size $B$ to be 6. We consider two cases: 
\begin{itemize}
\item The non overlapping case (1), and
\item The partial overlapping case (2)
\end{itemize}

Fig. \ref{fig:boundmse} shows the probability of error as the number of measurements increases. In the figure, note that we show the total number of groups $(M)$ in the signal. For each $M$, we fix the group sparsity level $k$ to be $M/10$. The results are averaged over 100 tests, and the probability of error is computed empirically. It can be seen in Fig. \ref{fig:boundmse} that, regardless of the groups overlapping or not, we need roughly the same number of measurements to achieve a low probability of error. 

\begin{figure}[!h]
\begin{center}
\subfigure[Non overlapping groups (i)]{
\includegraphics[scale = 0.38]{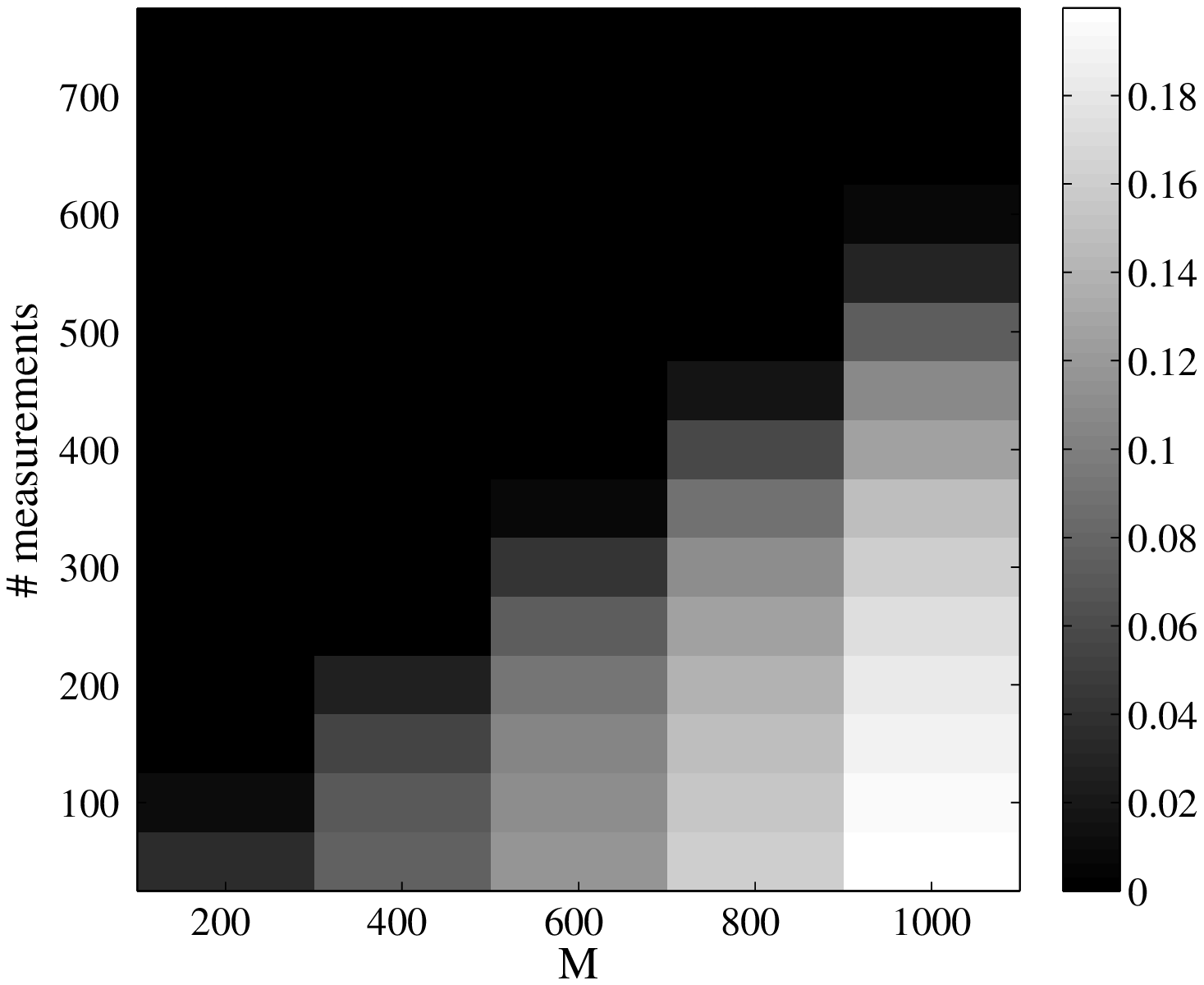}
\label{fig:t1}}
\subfigure[Overlapping groups (ii)]{
\includegraphics[scale = 0.39]{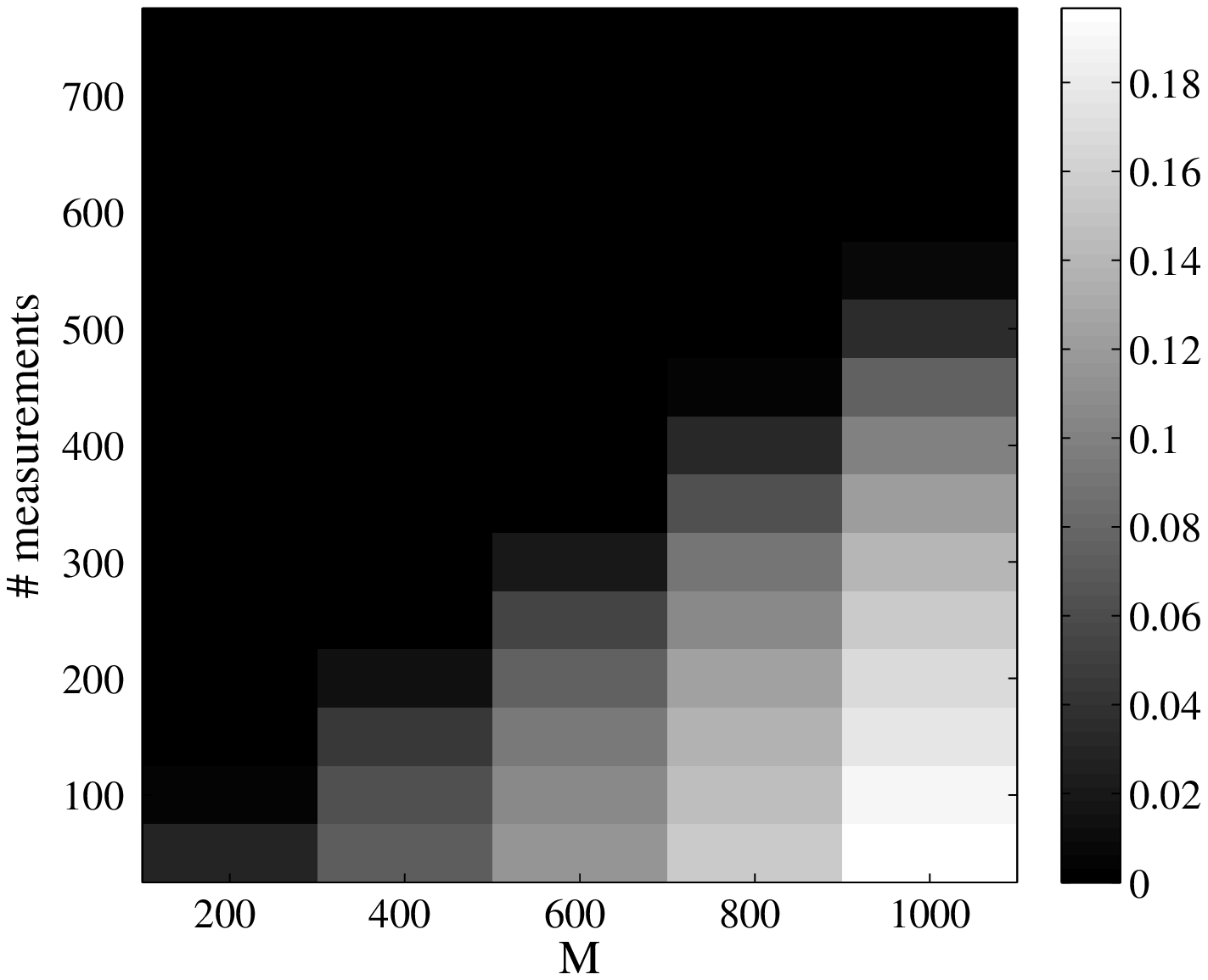}
\label{fig:t2}}
\caption{Number of groups \emph{vs} Number of measurements for exact signal reconstruction. The color bar indicates the probability of error, computed empirically. Notice how the number of measurements needed to achieve low probability of error is nearly the same in both cases, highlighting that we indeed do not pay a penalty for complicated grouping strategies}
\label{fig:boundmse}
\end{center}
\end{figure}

\subsection{Modeling DWT coefficients into groups}

In the spirit of \cite{som11, modelbased}, we considered a $128 \times 128$ section of the cameraman image, and obtained 5000 iid gaussian measurements from it. No noise was added to the image. We compare our methods with the ones displayed in \cite{som11}.  Fig \ref{fig:som} has been taken directly from \cite{som11}, and our result is shown in Fig. \ref{fig:ours}. 
\begin{figure*}[!ht]
\begin{center}
\subfigure[Reconstruction Performance of Various Methods]
{\includegraphics[scale = 0.5]{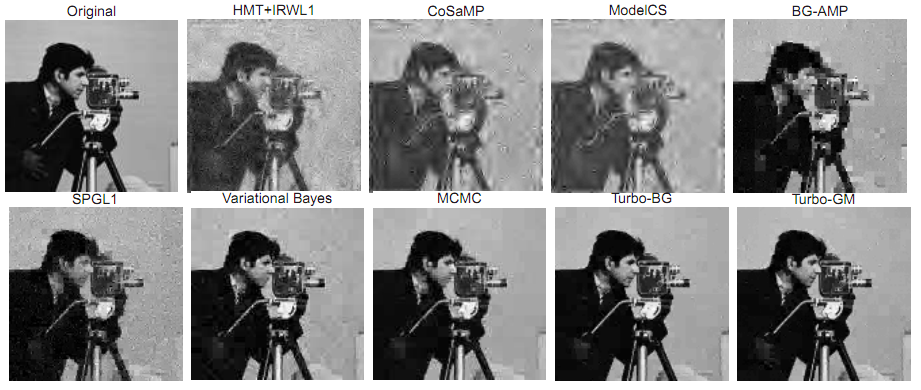}
\label{fig:som}}
\subfigure[Reconstruction using Glasso]
{\includegraphics[scale = 0.7]{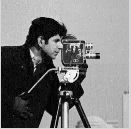}
\label{fig:ours}}
\caption{Reconstruction of a section of the cameraman image using various methods.}
\label{ampfigs}
\end{center}
\end{figure*}

For the basis of comparison, we zoom into similar regions from the best performing methods in Fig. \ref{fig:som} and Fig. \ref{fig:ours} , in Fig. \ref{zoom}. It can be seen that our method performs comparably to the turbo-BG and turbo-GM methods. 

\begin{figure}[!h]
\begin{center}
\subfigure[MCMC]
{\includegraphics[width = 40mm, height = 40mm]{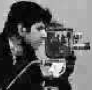}}
\subfigure[Turbo-BG]
{\includegraphics[width = 40mm, height = 40mm]{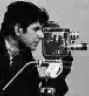}}
\subfigure[Turbo-GM]
{\includegraphics[width = 40mm, height = 40mm]{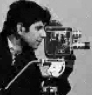}}
\subfigure[Glasso]
{\includegraphics[width = 40mm, height = 40mm]{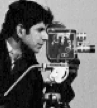}}
\caption{comparison of our method against the state of the art. The figures shown are zoomed in versions of those in Fig. \ref{ampfigs}}
\label{zoom}
\end{center}
\end{figure}

Along similar lines, we tested our methods for noiseless image recovery using the Microsoft Research Object Class Recognition  database\footnote{http://research.microsoft.com/en-us/projects/ObjectClassRecognition}. The dataset consists of images categorized into 20 types, with roughly 30 images of each type. We used the first 10 images of each type to generate a training set of 200 images, which was used to learn the regularization parameters. To compare and contrast our results with other methods tested in \cite{som11} (Fig. 6)\footnotemark, we compute the Normalized Mean Square Error (NMSE) of our methods over the same dataset for comparison. The normalized mean square error (in dB) for the true image $\vx$ is given by $10\times \log\left( \frac{\|\widehat{\vx} - \vx\|^2}{\|\vx\|^2} \right)$. We resized the images to size $128 \times 128$, and obtained $5000$ measurements for each case. 

\begin{figure*}[!htp]
\centering
\includegraphics[scale = 0.45]{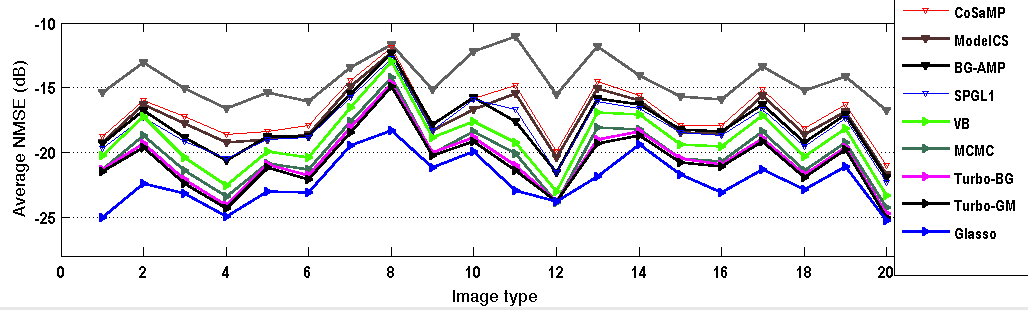}
\caption{comparison of various methods, and Glasso. Note that, the lower the value of NMSE, the better the performance.}
\label{fig:comparison}
\end{figure*}

\footnotetext{The authors thank Subhojit Som and Phil Schniter for sharing data for Fig. \ref{fig:comparison}}

\begin{figure}[!!h]
\centering
\includegraphics[scale = 0.5]{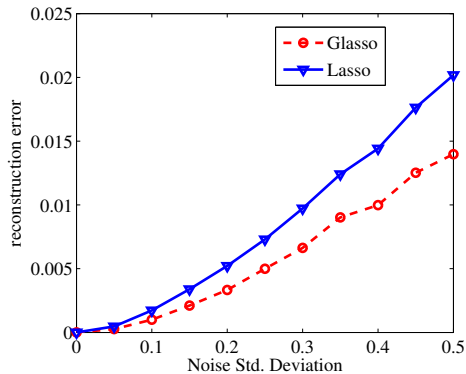}
\caption{Comparison of the two methods in the presence of noise.}
\label{noisy}
\end{figure}
Fig. \ref{noisy} shows the results we obtain as a function of the noise standard deviation. For the purpose of the experiment, we consider piecewise constant signals of length 1024, having 5 jumps. The location of the jumps is chosen at random, and the magnitude of each ``piece" is chosen uniformly between $[-1,1]$. We take 256 measurements for both the lasso and group lasso. From the figure, it is clear that by modeling the wavelet coefficients into parent-child pairs, we can better reconstruct signals in the presence of noise. The results are averaged over 1000 randomly generated signals.

\section{Conclusions and Discussion}
\label{sec:conc}
In this paper, we showed that one can  recover a signal known to lie in a  sparse union of subspaces exactly using atomic norm minimization algorithms. The number of measurements needed depend on the number of active subspaces, their dimensions and the relative angles between subspaces. We also showed that the measurement bound we derived is universal, in that it holds regardless of the nature and specific structure of the subspaces, in terms of overlaps. Indeed, the bounds can be specialized to cases where one is interested in a specific type of grouping. We subsequently extended these results for signals lying in a generic union of groups.

We also proposed a novel modeling strategy for DWT coefficients of signals. By modeling the coefficients into parent-child groups, we were able to take advantage of convex recovery methods that provably guarantee signal recovery. We showed that our method is at least as good as the current state of the art in non adaptive compressed sensing.



We note here that we do not claim the optimality of the particular grouping method that we have used , \emph{viz.} grouping the parent child pairs together. How best to group wavelet coefficients is still an open question, and is an avenue for further research. We prefer the parent-child pairs for its simplicity, and due to the fact that the groupings yield acceptable results, as seen in Section \ref{sec:expts}. 


%

\section{Appendices}
\subsection{Proof of Lemma \ref{lemma:chisquos}}
\label{appa}
\begin{proof}
Let $M_L := \max_{1\leq i \leq L} q_i$. For $t>0$, we have that
\begin{align*}
\E[M_L] &= \frac{\log[\exp(t\cdot \E[M_L])]}{t} \\
               &\stackrel{(a)}\leq \frac{\log[\E[\exp(t \cdot M_L)]]}{t} \\
               &\stackrel{(b)}= \frac{\log[\E [\max_{1\leq j \leq L} \exp(t \cdot q_j)]]}{t} \\
               &\stackrel{(c)} \leq \frac{\log[L \E[\exp(t \cdot q_1)]]}{t} \\
               &= \frac{\log(L) - \frac d2 \log(1 - 2t)}{t} 
\end{align*}
Where (a) follows from Jensen's inequality , (b) follows from the monotonicity of the exponential function, and (c) merely bounds the maximum by the sum over all the elements. Now, setting $t = [(2 + 2\epsilon)]^{-1}$ with $\epsilon = \sqrt{\frac{d}{2\log(L)}}$ yields 
\[
\E[M_L] \leq (\sqrt{2\log(L)} + \sqrt{d})^2
\]
\end{proof}

Note that $t$ can be optimized depending on the application.    We use this particular choice because it makes no assumptions about the relative magnitudes of $(M-k)$ and $B$.

\subsection{Proof of Lemma \ref{lemma:ballbound}}
\label{appb}
\begin{proof}
Let $K^\star = [K_j]_{j \in J^\star}$. By duality, it suffices to show that $\|z\|_{\A} \leq \sqrt{|\K^\star|}~\  \sigma_p(\mtx{K}^\star) \|z\|$ for all $z$ with $supp(z) \subset \K^\star$.  For any such $z$, there exists a representation $z = \K^\star b$ such that 
\begin{align*}
z &= [K_{j1} K_{j2} \ldots K_{j|J|}] [b^1 b^2 \ldots b^{|\K^*|}]^T \\
&= \sum_{K \in \K^\star} K_i b^i 
\end{align*}
so that none of the supports of $b^i$ overlap.  It then follows that
	\begin{align*}
	\|z\|_{\A} &= \| \sum_{K \in \K^\star} K_i b^i \| \\
	&\stackrel{(i)}\leq \sum_{K\in \K^\star} \|b^i\| \\ 
	&\stackrel{(ii)}\leq \sqrt{|\K^\star|} \left(\sum_{K\in \K^\star} \|b_K\|^2\right)^{1/2} \\  
	&= \sqrt{|\K^\star|} \|b\| \\
	&\stackrel{(iii)}\leq \sqrt{|\K^\star|} ~\ \sigma_p(\mtx{K}^\star) ~\ \|\z\|
	\end{align*}
	Where (i) follows from the definition of the norm  $\|\cdot\|_\A$, (ii) is a consequence of the relation $\|\beta\|_1 \leq \sqrt{k} \| \beta \|_2$ for $k$ dimensional vectors $\beta$ and (iii) follows from minimizing $\|\vct{b}\|$ subject to $\z = \mtx{K}^\star \vct{b}$
\end{proof}

\subsection{Bounding the last term in the proof of Theorem \ref{mainTh}}
\label{appc}
\begin{proof}
First, note that we can write 
\begin{align*}
\E \left[ \| \mk^S \bw^S \|^2 \right]  &= \E \left[ \| \mk \mP \bw \|^2  \right] \\
&=  \E \left[ \| \mk \mP\mk^T  (\mk\mk^T)^{-1} \vw \|^2\right] 
\end{align*}
Where $\mP(\cdot)$ is the operator that projects $(\cdot)$ onto the space spanned by the indices in $S$. 
 
Now, 
\begin{align*}
& \E \left[ \| \mk \mP\mk^T  (\mk\mk^T)^{-1} \vw \|^2\right]  \\
&= \E \left[  \| \vw^T (\mk\mk^T)^{-1}\mk \mP^T\mk^T \mk \mP\mk^T  (\mk\mk^T)^{-1} \vw \|^2\right] \\
&= \E \left[ tr \left(\mk \mP\mk^T  (\mk\mk^T)^{-1} \vw \vw^T (\mk\mk^T)^{-1}\mk \mP^T\mk^T  \right) \right] \\
&= tr\left( \mk \mP\mk^T  (\mk\mk^T)^{-1} \E[\vw \vw^T]  (\mk\mk^T)^{-1}\mk \mP^T\mk^T  \right) \\
&\stackrel{(a)}{=} tr\left( \mk \mP\mk^T  (\mk\mk^T)^{-1} (\mk\mk^T)^{-1}\mk \mP^T\mk^T  \right) \\
&= \| (\mk \mk^T)^{-1} \mk \mP \mk^T \|^2_F\\
\end{align*} 
where (a) is a consequence $\vw$ being a standard Gaussian vector. Now, let the singular value decomposition of $\mk$ be $\mk = \mtx{U} \mtx{\Sigma} \mtx{V^T}$. Then, 
\begin{align*}
& \| (\mk \mk^T)^{-1} \mk \mP \mk^T \|^2_F\\
&= \| \left( \mU \mS^2 \mU^T \right)^{-1} \mU \mS \mV^T \mP \mV \mS \mU^T \|^2_F \\
&= \| \mU^{-T} \mS^{-2} \mS \mV^T \mP \mV \mS \mU^T \|^2_F \\
&= \| \mS^{-1} \mV^T \mP \mV \mS \|^2_F \\
&= \| \mS^{-1} \|^2 \| \mV^T \mP \mV \|^2_F \| \mS \|^2 \\
&= \kappa(\mk)^2 \| \mP \|^2_F \\
&= \kappa(\mk)^2 |S| 
\end{align*}

\end{proof}

\subsubsection*{Acknowledgments}
This work was partially supported by AFOSR grant FA9550-09-1-0140 and the DARPA KECOM Program, and by ONR Award N00014-11-1-0723.

{\small
\bibliographystyle{abbrv}
\bibliography{TSP_UoS_arxiv}
}

\end{document}